\PassOptionsToPackage{colorlinks,
            linkcolor=FireBrick,
            anchorcolor=red,
            citecolor=RoyalBlue}{hyperref}
\documentclass[11pt]{amsart}
\usepackage[svgnames]{xcolor}

\usepackage{hyperref}
\usepackage{graphicx, amsmath,fullpage, amssymb, amsthm, amsfonts, color, mathrsfs,  moreverb, mathtools, float}

\usepackage[numbers]{natbib}
\usepackage{enumerate}

\newtheorem{theorem}{Theorem}[section]
\newtheorem{lemma}[theorem]{Lemma}
\newtheorem{prop}[theorem]{Proposition}
\newtheorem{corollary}[theorem]{Corollary}

\theoremstyle{definition}
\newtheorem{definition}[theorem]{Definition}

\theoremstyle{remark}
\newtheorem{remark}[theorem]{Remark}
\numberwithin{equation}{section}

\usepackage{algorithm}
\usepackage{algorithmic}

\newcommand{\ER}{Erd\H{o}s-R\'enyi~}
\newcommand{\R}{\mathbb R}

\newcommand{\1}{\mathbf 1}

\newcommand{\mask}{M}

\ifpdf
  \DeclareGraphicsExtensions{.eps,.pdf,.png,.jpg}
\else
  \DeclareGraphicsExtensions{.eps}
\fi

\usepackage{enumitem}
\setlist[enumerate]{leftmargin=.5in}
\setlist[itemize]{leftmargin=.5in}

\begin{document}

\title{Deterministic tensor completion with\\
hypergraph expanders}
\author{Kameron Decker Harris}
\address{Department of Computer Science,
Western Washington University, Bellingham, WA 98225}
\email{kameron.harris@wwu.edu}

\author{Yizhe Zhu}
\address{Department of Mathematics, University of California, San Diego, La Jolla, CA 92093}
\email{yiz084@ucsd.edu}

\keywords{hypergraph expander, tensor completion, max-quasinorm}
\date{\today}


\maketitle

\begin{abstract}  
 We provide a novel analysis of low-rank tensor completion based on hypergraph expanders. 
As a proxy for rank, we minimize the max-quasinorm of the tensor, which generalizes the max-norm for matrices.
Our analysis is deterministic and shows that the number of samples required to approximately recover an order-$t$ tensor with at most $n$ entries per dimension is linear in $n$, under the assumption that the rank and order of the tensor are $O(1)$.
As steps in our proof, we find a new expander mixing lemma for a $t$-partite, $t$-uniform regular hypergraph model, and prove several new properties about tensor max-quasinorm. To the best of our knowledge, this is the first deterministic analysis of tensor completion. We develop a practical algorithm that solves a relaxed version of the max-quasinorm minimization problem, and we demonstrate its efficacy with numerical experiments.
\end{abstract}

\section{Introduction}

\subsection{Matrix and tensor completion}

Classical compressed sensing considers 
the recovery of
high-dimensional structured signals from a
small number of samples. 
These signals are typically represented by
sparse vectors or low rank matrices. 
A natural generalization is to study recovery of
higher-order tensors,
i.e.\ a multidimensional array of real numbers with more than two indices,
using similar low rank assumptions.
However, much less is understood about compressed sensing of tensors.

Matrix completion is the problem of reconstructing a matrix
from a subset of entries, 
leveraging prior knowledge such as its rank.
The sparsity pattern of observed entries
can be thought of as the adjacency or biadjacency matrix of a graph,
where each edge corresponds to an observed entry in the matrix.
There are two general sampling approaches studied for matrix completion.
During probabilistic sampling, 
the entries in the matrix are observed at random
according to either Erd\H{o}s-R\'{e}nyi 
\citep{keshavan2010matrix,recht2011simpler}, 
random regular bipartite
\citep{gamarnik2017matrix,brito2018spectral}
or more general graph models \citep{klopp2014noisy}. 
Deterministic sampling, 
on the other hand, studies precisely what kind of graphs
are good for matrix completion and offers some advantages: 
One does not have to sample different entries for new matrices,
and any recovery guarantees are deterministic without failure probability. 
It has been shown in
\cite{heiman2014deterministic,bhojanapalli2014universal,burnwal2020deterministic}
that expander graphs, which have pseudo-random properties,
are a good way to sample deterministically for matrix completion. 
A deterministic theory of matrix completion based on graph limits,
a different approach,
appeared very recently \citep{chatterjee2020deterministic}.

Tensor completion, 
in which we observe a subset of the entries in a tensor
and attempt to fill in the unobserved values,
is a useful problem with a number of data science applications
\citep{liu2012tensor,song2019tensor}.
But fewer numerical and theoretical
linear algebra tools exist
for working with tensors than for matrices. 
For example, computing the spectral norm of a tensor, 
its low rank decomposition, 
and eigenvectors all turn out to be NP-hard \citep{hillar2013most}. 

Let the tensor of interest $T$ be order-$t$, each dimension of size $n$,
and have $\mathrm{rank}(T)=r$,
i.e.\ $T \in \bigotimes_{i=1}^t \R^n$ 
(we introduce our notation more fully in Section~\ref{sec:notation}). In this paper, tensor rank will always be using the canonical
polyadic (CP) decomposition \citep{kolda2009tensor}.
A fundamental question in tensor completion is how many observations via uniform sampling
are required to guarantee recovery of a unknown tensor with high probability. 
A naive lower bound for the sample complexity is $\Omega(n r t)$, the number of unknown parameters in a CP decomposition.
 
Compared to the classical matrix completion problem, an important phenomenon in tensor completion is the trade-off between computational and statistical complexity. 
One way to reduce tensors to matrices is by 
flattening the order $t$-tensor into a $n^{\lceil t/2\rceil}\times n^{\lceil t/2 \rceil}$ matrix, which can be solved in polynomial time.
Yet current results using flattening have sample complexity at best
$O \left({r n^{\lceil t/2\rceil}}\right)$
\citep{gandy2011tensor,mu2014square}.  A different approach using nuclear norm minimization  was studied in \cite{yuan2016tensor} with sample complexity $\tilde{O}(n^{t/2})$, where $\tilde{O}(\cdot)$ hides polylog factors. 
However, it is shown in \cite{friedland2018nuclear} that computing the nuclear  norm of a given tensor is NP-hard. 
The best known polynomial time algorithms require $O(n^{t/2})$ sample complexity for an order-$t$ tensor including spectral algorithms \cite{montanari2018spectral,xia2019polynomial}, gradient descent \cite{cai2019nonconvex,xia2019polynomial}, alternating minimization \cite{jain2014provable,liu2020tensor} convex relaxation via sum
of squares \cite{barak2016noisy,potechin2017exact}, or iterative collaborative filtering \cite{shah2019iterative}. 
There is still a huge gap between the sample complexity of the existing polynomial algorithms and the statistical lower bound. 
In \cite{barak2016noisy}, Barak
and Moitra conjectured that for an order-3 tensor, $\Omega(n^{3/2})$ many samples are needed for any polynomial algorithms by connecting
it to the literature on refuting random 3-SAT. 
All of the above results are concerned with uniform sampling. With adaptive sampling, $O(n)$ sample complexity was obtained in \cite{krishnamurthy2013low,zhang2019cross}. Very recently, Yu provided an algorithm that estimates a subclass of low-rank tensor with nearly linear samples \cite{yu2020tensor}.

In \cite{ghadermarzy2018}, Ghadermarzy, Plan, and Yilmaz
studied tensor completion without reducing it to a matrix case
by minimizing a max-quasinorm 
(satisfying all properties of the norm except a modified triangle
inequality, which we call the ``max-qnorm'')
as a proxy for rank.
This is defined as
\begin{equation*}
    \| T \|_\mathrm{max} =
    \min_{T = U^{(1)} \circ \cdots \circ U^{(t)}}
    \prod_{i=1}^t \| U^{(i)} \|_{2,\infty} \; ,
\end{equation*}
where the factorization is a CP decomposition of $T$
(see Definition~\ref{def:max-qnorm} for further details).
This is a generalization of the max-norm for matrices that
many have shown yields good matrix completion results
\citep{srebro2005rank,rennie2005fast,foygel2011,heiman2014deterministic,cai2016,foucart2017}.
Assuming that the observed entries are sampled 
from some probability distribution,
it was shown that solving a max-qnorm constrained least-squares problem results in
$ O \left( \frac{nt}{\varepsilon^2} \right)$ 
sample complexity when $r = O(1)$,
and even faster rates in $\varepsilon$ for the case of zero noise
\citep{ghadermarzy2018}. 
However, it is not clear if minimizing the max-qnorm is NP-hard.

We study the deterministic analog of the tensor completion approach proposed in \cite{ghadermarzy2018}.
The deterministic analysis leads to a sample complexity which is also linear in $n$,
albeit with weaker dependence on other parameters
(see Section~\ref{sec:main-sample-complexity}).
We assume that the observed entries correspond to the
edges in an expander hypergraph. It has been known that revealing entries of a low-rank matrix according to the edges of an expander graph allows matrix completion with small errors \cite{heiman2014deterministic,bhojanapalli2014universal}. To the best our knowledge, our work is the first generalization of such connection to hypergraph expanders and tensor completion.
A deterministic algorithm on low-rank tensor approximation was studied recently in \cite{musco2019low} based on multiparty communication complexity. However, as pointed out by the authors of \cite{musco2019low}, their analysis cannot be applied to tensor completion problems directly. 

\subsection{Expanders and mixing}

The expander mixing lemma for $d$-regular graphs 
\citep[e.g.][]{chungspectral} 
states the following: 
Let $G$ be a $d$-regular graph on $n$ vertices with  
$\lambda=\max \{\lambda_2,|\lambda_n|\}<d$.
For any two sets $V_1, V_2 \subseteq V(G)$, let
$
e(V_1,V_2)=|\{(x,y)\in V_1\times V_2: xy\in E(G) \}|
$
be the number of edges between $V_1$ and $V_2$.
Then we have that
\begin{align}\label{eq:expandermixing}
    \left| e(V_1,V_2)-\frac{d|V_1||V_2|}{n}\right| \leq \lambda \sqrt{|V_1| |V_2|\left(1-\frac{|V_1|}{n}\right) \left(1-\frac{|V_2|}{n}\right)}.
\end{align}
Equation~\eqref{eq:expandermixing}
tells us regular graphs with small $\lambda$ have the {\em expansion property},
where the number of edges between any two sets is well-approximated 
by the number of edges we would expect if they were drawn at random.
The quality of such an approximation is controlled by $\lambda$. 
It's known from the Alon-Boppana bound that 
$\lambda\geq 2\sqrt{d-1}-o(1)$,
and regular graphs that achieve this bound are called Ramanujan.  Deterministic and random constructions of Ramanujan 
(or nearly-so) graphs
have been extensively studied
\citep{ben2011combinatorial,marcus2015interlacing,bilu2006lifts,burnwal2020deterministic,mohanty2020explicit}.

Higher order, i.e.\ hypergraph,
expanders have received significant attention in 
combinatorics and theoretical computer science 
\citep{lubotzky2017high}.
There are several expander mixing lemmas in the literature
based on spectral norm of tensors
\citep{friedman1995second,parzanchevski2017mixing,cohen2016inverse,lenz2015eigenvalues}. An obstacle to applying such results to tensor completion is that
in most cases the second eigenvalues of tensors are unknown, even approximately.
In \cite{dumitriu2019spectra}, 
an expander mixing lemma similar to 
\eqref{eq:expandermixing} based on the second eigenvalue of 
the adjacency matrix of regular hypergraphs was derived.
However, for our application
we need an expander mixing lemma that estimates the 
number of hyperedges among $t$ different vertex subsets. 

One exception is the work of
\cite{friedman1995second},
who studied a $t$-uniform hypergraph model on $n$ vertices with 
$d n^{t-1}$ hyperedges chosen randomly with $d\geq C\log n$. 
They proved that the second eigenvalue of the associated tensor
$\lambda = O((\log n)^{t/2}\sqrt{d})$. 
However, this is  a relatively dense random hypergraph model,
since the number of edges grows superlinearly with $n$ for $t>2$.
Thus Friedman and Widgerson's
model only applies when one has the ability to make many measurements,
as opposed to the more realistic ``big data'' scenario
constrained to $\tilde {O}(n)$ observations.
If we sample the original tensor according to the 
hyperedge set of a hypergraph, 
we would like the number of hyperedges to be small,
in order to be able to represent a small number of samples.
From previous results on matrix completion (\cite{heiman2014deterministic,bhojanapalli2014universal}),
we expect the reconstruction error 
should be controlled by a parameter that is 
related to the expansion property of the hypergraph. 

Our tensor completion analysis can also be applied to observation models where entries are revealed according to a general $t$-uniform hypergraph. Define the {\bf spectral norm} of a tensor $T$ as
\begin{align}\label{eq:defspecT}
    \|T\|=\sup_{v_1,\dots,v_k\in S^{n-1}} \left|\sum_{i_1,\dots,i_k=1}^n T_{i_1,\dots,i_k}v_1(i_1)\cdots v_k(i_k)\right|,
\end{align}
where $S^{n-1}$ is the unit sphere in $\mathbb R^n$.
The mixing lemma for any $t$-uniform hypergraph was first considered in \cite{friedman1995second} and was generalized in \cite{cohen2016inverse}. We state the expander mixing lemma for $t$-partite $t$-uniform hypergraphs as follows. A quick proof of Lemma \ref{lem:tensormixing} is provided in Appendix \ref{sec:lemmix}.
\begin{lemma}\label{lem:tensormixing}
Let $H=(V,E)$ be a $t$-uniform $t$-partite hypergraph with vertex set $V=V_1\cup \cdots \cup V_t$ and  adjacency tensor $T_H$. For any subsets $W_1\subset V_1$,\dots,$W_t\subset V_t$, define 
\begin{align*}
    e(W_1,\dots,W_t)=\left| \left\{ e=(v_1,\dots, v_t)\in E, v_i\in W_i, 1\leq i\leq t \right\}  \right|.
\end{align*}
Then the  following holds
\begin{align}\label{eq:Hmixinglemma}
    \left|~e(W_1,\dots,W_t)-\frac{|E|}{n^t}|W_1|\cdots |W_t|~\right|\leq \lambda_{2}(H)\sqrt{|W_1|\cdots |W_t|},
\end{align}
where
\begin{align}\label{eq:lambda2H}
    \lambda_{2}(H)=\left\|T_H-\frac{|E|}{n^t} J\right\|,
\end{align}
and $J$ is the all-ones tensor.
\end{lemma}

\subsection{Main results}\label{sec:main}
In this paper we seek the connection between two topics: hypergraph expanders and tensor completion,  using the tensor max-quasinorm  introduced in \cite{ghadermarzy2018}.

We  revisit and generalize the sparse, deterministic hypergraph construction
introduced in \cite{bilu2004codes}.
We construct a $t$-uniform $t$-partite hypergraph
by taking a $d$-regular ``base'' graph and 
forming a hypergraph from its walks of length $t$.
In this model, each node is of degree $d^{t-1}$,
corresponding to $nd^{t-1}$ samples. 
An advantage of our expander mixing result 
is that the expansion property of the hypergraph is
controlled by the expansion in a $d$-regular graph (Theorem \ref{thm:mixing}).
This is easy to compute and optimize
using known constructions of $d$-regular expanders. Based on such hypergraphs, 
we perform a deterministic analysis of
an optimization problem similar to that analyzed by
\citet{ghadermarzy2018}. 

Our main contributions can be summarized as follows:
First, we obtain a variant of the expander mixing inequalities  from \cite{bilu2004codes,alon1995derandomized} and generalize the $t$-partite $t$-uniform regular hypergraph construction in \cite{bilu2004codes} to  $t$-partite $t$-uniform quasi-regular hypergraphs, see Section \ref{sec:construction} and Section \ref{sec:expander}. 
The new expander mixing result provides a better error control for tensor completion. 
This improvement might also be useful for the  application of hypergraph codes studied in \cite{bilu2004codes}.
Next, we perform a deterministic analysis of an optimization problem similar to that analyzed by \citet{ghadermarzy2018}. 
Our proof is based on the techniques used
 to study matrix completion in 
 \cite{heiman2014deterministic}
 (see also \cite{brito2018spectral}).
We prove several useful linear algebra facts
about the max-quasinorm for tensors
in order to prove the main results on the tensor completion error, 
but which may be of separate interest.
Finally, we show proof-of-concept numerical results on minimal max-quasinorm completion.

 For a deterministic hypergraph $H$, if we have a good estimate of $\lambda_2(H)$ defined in \eqref{eq:lambda2H}, we can  obtain the following tensor completion error bound.  Here we state it for $t$-uniform $t$-partite hypergraph to avoid symmetric sampling, which might be wasteful and will increase the sample size by a factor of $t!$, but the analysis can be extended to general $t$-uniform hypergraphs.
 
\begin{theorem}\label{thm:main3}
Given a hypercubic tensor $T$ of order $t$, 
reveal its entries according to a  $t$-uniform  $t$-partite hypergraph
$H=(V,E)$ with $V=V_1\cup \cdots \cup V_t$, $|V_1|=\cdots=|V_t|=n$, and second eigenvalue $\lambda_2(H)$.
Then solving 
\[
\hat{T} = \arg \min_{T'} \| T' \|_\mathrm{max}
\mbox{\quad such that \quad
$T'_{e} = T_{e}$
\quad for all\quad  $e \in E$}
\]
will result in the following error bound:
\begin{align} \label{eq:errorbound15}
\frac{1}{n^t}
\|\hat{T}-T\|_F^2
\leq
\frac{2^{2t}n^{t/2}K_G^{t-1}\lambda_2(H)}{|E|}\|T\|_{\max}^2,
\end{align}
where $K_G\leq 1.783$ is  Grothendieck's constant over $\mathbb R$.
\end{theorem}

Although it's NP hard to compute $\lambda_2(H)$ for a given hypergraph $H$, in \cite{friedman1995second,jain2014provable,cohen2016inverse,zhou2019sparse}, the upper bound on $\lambda_2(H)$ is provided for some random hypergraph models. Let $T_H$ be the adjacency tensor of a given $H$. From the definition of $\lambda_2(H)$ and \eqref{eq:defspecT}, suppose $H$ has  $|E|=o(n^t)$ many hyperedges, we have \[\lambda_2(H)\geq \max_{i_1,\dots,i_t} \left|{(T_{H})}_{i_1,\dots,i_t}-\frac{|E|}{n^t}\right|=\Omega(1). \]
Therefore, if we want to use \eqref{eq:errorbound15} to control the mean square error by $\varepsilon$, the number of samples must be $\Omega (n^{t/2})$. In \cite{jain2014provable}, an estimate of $\lambda_2(H)$ for 3-uniform \ER random hypergraphs with $p=\frac{c\log n}{n^{3/2}}$ was obtained and our Theorem \ref{thm:main3} can be applied.

However, using a hypergraph expander model based on regular graphs, we can get a better error bound without using $\lambda_2(H)$.
We  state our main results on the deterministic bounds for tensor completion based on hypergraph expander models.  A formal definition of this hypergraph model is given in Section \ref{sec:construction}. When $t=2$, it reduces to the result in \cite{heiman2014deterministic}.
\begin{theorem}
\label{thm:main}
Given a hypercubic tensor $T$ of order-$t$, 
reveal its entries according to a 
$t$-partite, $t$-uniform,
$d^{t-1}$-regular hypergraph
$H=(V,E)$
constructed from a $d$-regular graph $G$ of size $n$
with second eigenvalue (in absolute value) $\lambda\in (0,d)$ (see Section \ref{sec:regular}).
Then solving 
\begin{align}\label{eq:opt}
\hat{T} = \arg \min_{T'} \| T' \|_\mathrm{max}
\mbox{\quad such that \quad
$T'_{e} = T_{e}$
\quad for all\quad  $e \in E$}
\end{align}
will result in the following mean squared error bound:
\begin{align}\label{eq:errorbound}
\frac{1}{n^t}
\|\hat{T}-T\|_F^2
\leq C_t
 \| T \|_\mathrm{max}^2 \frac{\lambda}{d},
\end{align}
where $C_t=2^{t} (2t-3)K_G^{t-1}$, and
$K_G\leq 1.783$ is  Grothendieck's constant over $\mathbb R$.
\end{theorem}

\begin{remark}
From Theorem \ref{thm:max-rank}, $\|T\|_{\max}^2\leq r^{t^2-t-1}|T|_{\infty}$, where $r$ is the rank of the tensor $T$. The number of revealed entries in $T$ is $nd^{t-1}$. The analysis of sample complexity for the algorithm is given in Section \ref{sec:main-sample-complexity}.
\end{remark}

If the tensor $T$ is not hypercubic, say $T\in \bigotimes_{i=1}^t \mathbb R^{n_i}$, we can still apply Theorem \ref{thm:main} by taking $n=\max_{1\leq i\leq t} n_i$ and embed $T$ in $\bigotimes_{i=1}^t \mathbb R^{n}$ by filling extra entries with zeros. But when $n_1,\dots, n_t$ are heterogeneous, this might be wasteful. In other to handle heterogeneous dimensions, we generalize the construction in \cite{alon1995derandomized,bilu2004codes} to construct $t$-uniform $t$-partite quasi-regular hypergraphs with  good expansion properties, based on bipartite biregular expanders, see Section \ref{sec:bipartiteextension}, which yields the following theorem. The proof of Theorem \ref{thm:main2} is similar to Theorem \ref{thm:main} and we include it in Section \ref{sec:proof_quasiregular_completion}.

\begin{theorem}
\label{thm:main2}
Given a tensor $T\in \bigotimes_{i=1}^t \mathbb R^{n_i}$, 
reveal its entries according to a 
$t$-partite, $t$-uniform,
quasi-regular hypergraph
$H=(V,E)$
constructed from a collection of  $(d_{2i-1},d_{2i})$-biregular bipartite graphs with second eigenvalues $\lambda^{(i)}$ for $1\leq i\leq t-1$ (see Section \ref{sec:quasiregular}).
Then solving 
\[
\hat{T} = \arg \min_{T'} \| T' \|_\mathrm{max}
\mbox{\quad such that \quad
$T'_{e} = T_{e}$
\quad for all\quad  $e \in E$}
\]
will result in the following mean squared error bound:
\begin{align*} 
\frac{1}{\prod_{i=1}^t n_i}
\|\hat{T}-T\|_F^2
\leq 2^{t}K_G^{t-1} \left(\frac{\lambda^{(1)}}{\sqrt{d_1d_2}}+\sum_{k=2}^{t-1}\frac{2\lambda^{(k)}}{\sqrt{d_{2k-1}d_{2k}}}\right)
  \|T \|_\mathrm{max}^2.
\end{align*}
\end{theorem}
When $t=2$, Theorem \ref{thm:main2} reduces to Theorem 24 in \cite{brito2018spectral} for deterministic matrix completion with bipartite biregular graphs.

The main motivation to study the mixing properties of hypergraph expanders in Section \ref{sec:construction} and \ref{sec:expander} is that the key parameter to control the mixing properties is the spectral gap, which is well understood in spectral graph theory and easy to compute. This gives us explicit error control in Theorem \ref{thm:main} and Theorem \ref{thm:main2}.

\subsection{Sample complexity}
\label{sec:main-sample-complexity} We focus on the sample complexity analysis for the hypergraph model we used in Theorem \ref{thm:main} and other models can be analyzed in a similar way.  Recall that the number of edges in $H$ is $n d^{t-1}$,
equal to the number of samples.
Suppose we have an expander graph $G$, where $\lambda=O( \sqrt{d})$.
In order to guarantee  the right hand side in \eqref{eq:errorbound} is bounded by $\varepsilon$,
Theorems~\ref{thm:main} and \ref{thm:max-rank}
say that, assuming $t = O(1)$, we require
\begin{align}\label{eq:Ebound}
   \displaystyle |E| = O(\|T\|_{\max}^{4t-4} \varepsilon^{-2(t-1)}n)
   =
   O\left( \frac{ n r^{2(t-1)(t^2-t-1)}}{\varepsilon^{2(t-1)}} \right) 
\end{align} samples, which is linear in $n$. 
The computations are shown in Section~\ref{sec:sample-complexity}. 


In Theorem~\ref{thm:main},
the dependence on rank is  exponential in $t$,
since the best known dependence of the max-qnorm on rank
is
$\| T \|_\mathrm{max} = O (\sqrt{r^{t^2-t-1}})$.
For matrices, our results have $|E|=O(nr^2/\varepsilon^2)$, which is the same sample complexity derived in \cite{heiman2014deterministic}.
A better understanding of the max-qnorm
for tensors may lead to better dependence on the rank.

The sample complexity in Theorem \ref{thm:main3} can be similarly analyzed. To guarantee an $\varepsilon$ mean squared error,   we have $|E|=O(\lambda_2(H)\|T\|_{\max}^2\varepsilon^{-1} n^{t/2})$. The dependence on $\|T\|_{\max}$ and $\varepsilon$ in this bound is better, but the dependence on $n$ is much weaker compared to \eqref{eq:Ebound}.

\subsection{Computational complexity}
It is not clear what is the computational complexity for solving the optimization problem \eqref{eq:opt}.
Solving it might be NP-hard, but we provide a practical algorithm to approximately solve it.

\subsection{Organization of the paper} 

In Section \ref{sec:construction},
we construct the  hypergraph expanders with good mixing properties. 
In Section \ref{sec:expander},
we prove an expander mixing lemma for such hypergraphs. 
In Section \ref{sec:property},
we prove several useful properties of max-quasinorm for tensors. In Section \ref{sec:completion},
we leverage these properties to analyze the above tensor 
completion algorithm and prove the main results.
We extend our result for tensor completion with errors in the observed entries,
which can model noise or adversarial corruptions. In Section \ref{sec:optimization} we provide a numerical  algorithm for finding tensors with the minimum complexity.
We conclude with a discussion of limitations and 
future directions in Section~\ref{sec:discussion}. 
Omitted proofs are provided in Appendix~\ref{sec:appendix}.

\subsection{Notation}
\label{sec:notation}

The notations we use throughout the paper come from the review
by \citet{kolda2009tensor}.
We use lowercase symbols $u$ for vectors,
uppercase $U$ for matrices and tensors.
The symbol ``$\circ$'' denotes the outer product of 
vectors, i.e.\ 
$T= u \circ v \circ w$
denotes the order-3, rank-1 tensor with entry 
$T_{i,j,k} = u_i v_j w_k$.
We also use this symbol for the outer product of matrices
as appears in the rank-$r$ decomposition
of a tensor $T = U^{(1)} \circ U^{(2)} \circ U^{(3)}$, 
where each matrix $U^{(i)}$ has $r$ columns,
so that 
$T_{i,j,k} = \sum_{l=1}^r U^{(1)}_{i,l} \, U^{(2)}_{j,l} \, U^{(3)}_{k,l}$,
and
$T = \bigcirc_{i=1}^t U^{(i)}$ 
is shorthand for the same order-$t$, rank-$r$ tensor.
The symbols $\otimes$ and $*$ denote
Kronecker and Hadamard products, respectively, which will be defined in Section \ref{sec:property}.
We use $\bigotimes_{i=1}^t \R^{n_i}$
for the space of all order-$t$ tensors with $n_i$ entries in the $i$-th 
dimension.
We use $1_A \in \R^n$ as the indicator vector of a set $A \subseteq [n]$,
i.e.\ $(1_A)_i = 1$ if $i \in A$ and 0 otherwise.  For any order-$t$ tensor $T\in  \bigotimes_{i=1}^t \R^{n_i}$ and subsets $I_i\subseteq [n_i]$, denote $T_{I_1,\dots,I_t}$ to be the subtensor restricted on the index set $I_1\times \cdots \times  I_t.$
Norms $\| \cdot \|$ are by default the $\ell_2$
norm for vectors and operator norm for matrices and tensors. 
We use the notation $| \cdot |_p$ for
entry-wise $\ell_p$ norms of matrices and tensors 
and always include the subscript to avoid confusion
with set cardinality.

\section{Construction of hypergraph expanders}
\label{sec:construction}

We start with the definition of a hypergraph and some basic properties.
\begin{definition}[hypergraph]
A \textbf{hypergraph} $H$ consists of a set of vertices $V$
and a set of hyperedges $E$,
where each hyperedge is a nonempty set of $V$, 
the vertices that participate in that hyperedge.
The hypergraph $H$ is $t$-\textbf{uniform} 
for an integer $t\geq 2$ 
if every hyperedge $e\in  E$ contains exactly $t$ vertices. 
The \textbf{degree} of vertex $i$ is the number of all hyperedges containing $i$.
A hypergraph is $d$-\textbf{regular} if all of its vertices have degree $d$.  

A $t$-uniform hypergraph is \textbf{$t$-partite} if its vertex set $V$ can be decomposed as $V_1\cup V_2\cup \cdots \cup V_t$ such that each hyperedge $e\in E$ consists of $t$ vertices $v_1,\dots, v_t$ such that $v_i\in V_i$ for $1\leq i\leq t$. For a $t$-uniform hypergraph is $t$-partite hypergraph $H=(V,E)$, we denote each hyperedge $e$ as an ordered tuple $(v_1,\dots,v_t)$ where $v_k\in V_k, 1\leq k\leq t$.
\end{definition}

\begin{definition}[adjacency tensor for $t$-uniform $t$-partite hypergraphs]
Let $H=(V,E)$ be a $t$-uniform $t$-partite hypergraph with vertex set $V=V_1\cup\cdots \cup V_t$ such that $|V_k|=n_k, 1\leq k\leq t$. We define the adjacency tensor $T\in \otimes_{k=1}^t \{0,1\}^{n_k}$  as 
\begin{align*}
    T_{i_1,\dots,i_t}= \begin{cases}
    1 & \text{if $(i_1,\dots,i_t)\in E$}, i_k\in V_k, 1\leq k\leq t, \\
    0 & \text{otherwise}.
    \end{cases}
\end{align*}
\end{definition}

\subsection{Construction of regular hypergraph expanders}\label{sec:regular}
Let $G=(V(G), E(G))$ be a connected $d$-regular 
graph on $n$ vertices
with second largest  eigenvalue (in absolute value)
$\lambda\in (0,d)$. 
We construct a $t$-partite, $t$-uniform, $d^{t-1}$-regular hypergraph 
$H=(V,E)$ from $G$ as follows.

\begin{definition}[regular hypergraph expander]\label{eq:def_regular}
Let $V=V_1\cup V_2\cup \cdots \cup V_t$ 
be the disjoint union of $t$ vertex sets
such that $|V_1| = \cdots = |V_t|=n$. 
The hyperedges of $H$ correspond to all walks  of length $t-1$ in $G$:
$(v_1 ,\dots, v_t)$ 
is a hyperedge in $H$ if and only if 
$(i_1,\dots, i_{t})$ is a walk of length $t-1$ in $G$.
\end{definition}


Given the description above, 
we have $|V|=nt$ and $|E|=n d^{t-1}$,
since $E$ contains all possible walks of length $t-1$ in $G$. 
Moreover, every vertex is contained in exactly $d^{t-1}$ many hyperedges,
so $H$ is regular.
From our definition of the hyperedges in $H$, 
the order of the walk in $G$ matters. 
For example, 
two walks $i_1 \to i_2 \to i_3$ and $i_3 \to i_2 \to i_1$ 
correspond to different hyperedges 
$(i_1,i_2,i_3)$ and $(i_3,i_2,i_1)$ in $H$ when $i_1\not=i_3$. When $t=2$, $H$ is a bipartite $d$-regular graph with $2n$ vertices.
See Figure~\ref{fig:hyp3} for an example of the construction with $t=3$.

\begin{figure}
\centering
\includegraphics[width=0.7\linewidth]{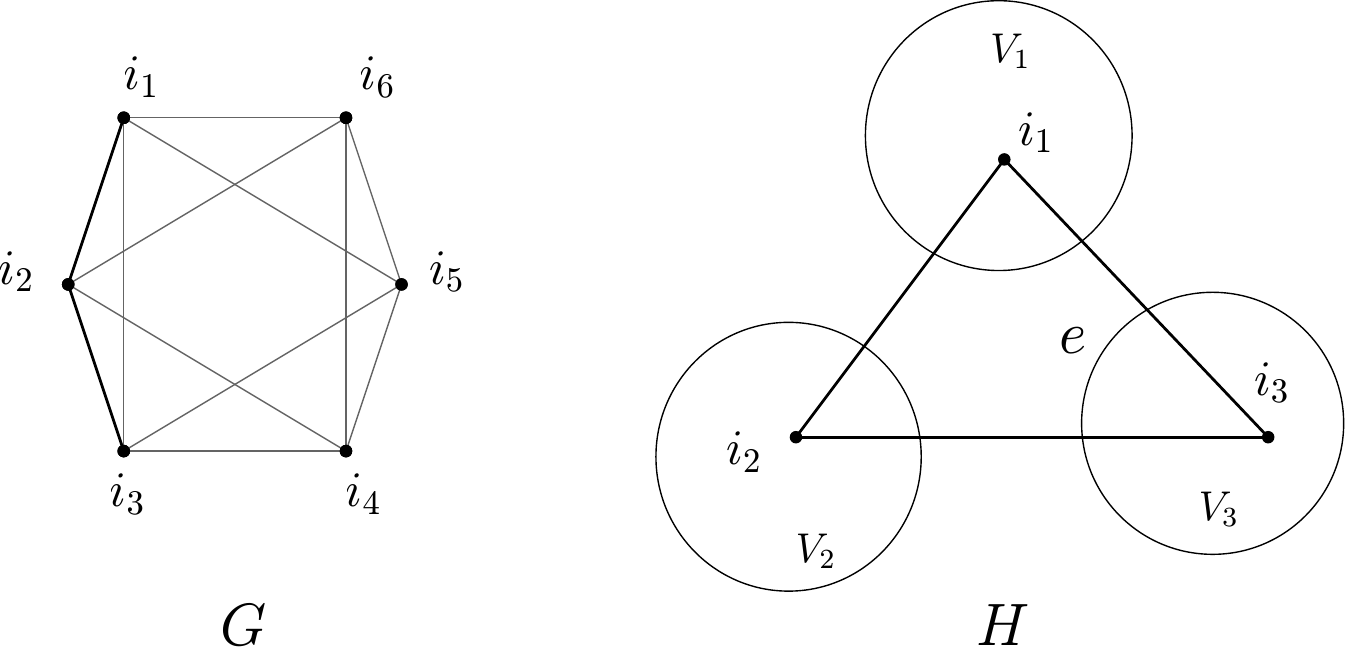}
\caption{An example hyperedge in the $t=3$ case: 
We depict the base graph $G$ on the left and a single edge in the hypergraph
$H$ on the right.
$(i_1, i_2, i_3)$
forms an hyperedge $e$ in $H$ if and only if 
$(i_1,i_2,i_3)$ is a walk in $G$.
}
\label{fig:hyp3}
\end{figure}

This construction was used by \cite{alon1995derandomized} to
de-randomize graph products and by \cite{bilu2004codes}
to construct error correcting codes. 
Both groups' results depended on analyzing the  expansion properties of this hypergraph model.

\subsection{Generalization to quasi-regular hypergraphs}\label{sec:quasiregular}
\begin{definition}[bipartite biregular graph]
A graph $G=(V,E)$ is \textbf{bipartite}  if the vertex set $V$ can be partitioned into disjoint vertex sets $V_1, V_2$ such that every edge connects a vertex in $V_1$ to a vertex in $V_2$. A bipartite graph $G$ is $(d_1,d_2)$-\textbf{biregular} if each vertex in $V_1$ has degree $d_1$ and each vertex in $V_2$ has degree $d_2$.
\end{definition}

We can also generalize our construction in Section \ref{sec:regular} to $t$-partite quasi-regular hypergraph $H$ with $|V_1|=n_1,\dots, |V_t|=n_t$. The idea is to combine $t-1$ many bipartite biregular graphs in the following way.

\begin{definition}[quasi-regular hypergraph expander]\label{eq:def_quasi}
Take $\{d_i\}_{1\leq i\leq 2t-2}$ such that $n_id_{2i-1}=n_{i+1}d_{2i}$ for $1\leq i\leq t-1$. Let $G_1,\dots, G_{t-1}$ be bipartite biregular graphs such that $G_i$ is $(d_{2i-1},d_{2i})$-biregular with two vertex sets of size $n_i$ and $n_{i+1}$, respectively. We construct $H$ such that an hyperedge $e=(v_1,\dots,v_t)$ is in $H$ if and only if $(v_i,v_{i+1})\in E(G_i)$ for all $1\leq i\leq t-1$. Now we have 
$|E(H)|=n_1\prod_{i=1}^{t-1}d_{2i-1}$. Denote the second largest eigenvalue of $G_i$ by $\lambda^{(i)}$.
\end{definition}

Note that the hypergraph $H$ is not regular, but each vertex in $V_i$ has the same degree for $1\leq i\leq t$, as shown in the following lemma. The proof of Lemma \ref{lem:quasiregular} is given in Section \ref{sec:lemdegree}.
\begin{lemma}\label{lem:quasiregular}
Let $H$ be the $t$-uniform $t$-partite quasi-regular hypergraph defined above,  then the degree of each vertex in $V_i$ is
$ \left(\prod_{k=1}^{i-1}d_{2k}\right)\left(\prod_{k=i}^{t-1}d_{2k-1}\right).$
\end{lemma}


Deterministic or random construction of bipartite biregular graph with a small second eigenvalue was considered in \cite{burnwal2020deterministic,brito2018spectral}, which can be used to construct the quasi-regular hypergraph in Definition \ref{eq:def_quasi}.


\section{Expander mixing}\label{sec:expander}

In this section, we prove  novel and tighter expansion properties in the hypergraph models we constructed in Section \ref{sec:construction}.
Later, we will apply them to tensor completion.





\subsection{Regular hypergraph expanders}
Let $G$ be a $d$-regular graph on $n$ vertices with 
 $\lambda\in (0,d)$,
and let $H$ be the corresponding $t$-partite, $t$-uniform
hypergraph constructed as in Section \ref{sec:regular}. 
We get the following mixing lemma for $H$. 
The mixing rate is essentially controlled by the second 
eigenvalue of the $d$-regular graph $G$. This is an advantage over other expander mixing lemmas for hypergraphs \cite{friedman1995second,cohen2016inverse,parzanchevski2017mixing}, since in our model the parameters that control the mixing rate are explicit and easy to compute.  

\begin{theorem} 
Given a base graph $G$ with $\lambda\in (0,d)$, form the hypergraph $H$ 
following the construction in Section~\ref{sec:construction}.
Let $W_i\subseteq V_i$,  $1\leq i\leq t$ be any non-empty subsets. Denote 
$
\alpha_i:=\frac{|W_i|}{n}\in [0,1],   1\leq i\leq t,
$
and let
\[
e(W_1,\dots, W_t):=\left|\{(v_1,\dots, v_t)\in W_1\times \cdots \times W_t: (v_1,\dots, v_t) \in E(H)\}\right|. 
\]
 Then the following expander mixing property holds:
\label{thm:mixing}
\begin{align}
    &\left|\frac{e(W_1,\dots,W_t)}{nd^{t-1}}-\prod_{i=1}^t\alpha_i\right| \notag\\
\leq & \frac{\lambda}{d} \left( \sqrt{\alpha_1(1-\alpha_1)\alpha_2(1-\alpha_2)}\prod_{k=3}^{t} \alpha_k+ \sum_{i=3}^{t} \sqrt{\alpha_1 \alpha_i(1-\alpha_i)} \prod_{k=i+1}^{t}\alpha_k\right)
\leq \frac{(2t-3)\lambda}{4d}.
\label{eq:mixing}
\end{align}
\end{theorem}
\begin{remark}
When $t=2$, Theorem \ref{thm:mixing} reduces to the expander mixing lemma \eqref{eq:expandermixing} for regular graphs.  The results in \cite{bilu2004codes, alon1995derandomized} estimate the ratio between $\frac{e(W_1,\dots,W_t)}{nd^{t-1}}$ and $\prod_{i=1}^t\alpha_i$. In particular, Lemma 3 in \cite{bilu2004codes} only provides  an upper bound,  and Theorem 4 of \cite{alon1986eigenvalues} gives a two-sided bound with an extra assumption that $\alpha_i>6\lambda/d$ for all $1\leq i\leq t$. Our new expander mixing result provides a two-sided control of the difference between the two quantities without any restriction on $\alpha_i$, which is important for our tensor completion analysis, and the results in \cite{alon1995derandomized,bilu2004codes} are not directly applicable. 
The ratio $\frac{e(W_1,\dots,W_t)}{nd^{t-1}}$ describes the hitting property of expander walks, which has many applications in theoretical computer science \cite{ajtai1987deterministic,kahale1995eigenvalues,hoory2006expander,cohen2020expander}.
\end{remark}

Let $A$ be the adjacency matrix of $G$ and $M=\frac{1}{d}A$ be the transition probability matrix for the simple random walk on $G$. Define $u=\frac{1}{ n} (1,\dots,1)^{\top}.$ Then $u$ is the first eigenvector of $M$ with eigenvalue $1$ and the second eigenvalue of $M$ is $\lambda/d$.  Define
\begin{align}
P_k&=\sum_{i\in W_k} e_ie_i^{\top}, \quad \1_{W_k}=\sum_{i\in W_k}e_i,
\end{align} 
where $e_i$ is the basis vector with $1$ in the $i$-th coordinate and $0$ elsewhere. 
Then $P_k$ is a projection matrix such that
$
    P_ke_i=\mathbf{1}\{i\in W_k \}e_i.
$
Let $P$ be the projection onto the orthogonal space of $u$. As shown in  \cite{alon1995derandomized}, considering the simple random walk of $t$ steps in $G$ with a uniformly chosen initial vertex,
the probability that the simple random walk stays in $W_i$ at step $i$ for all $1\leq i\leq t$ can be written as
\begin{align}
    \frac{e(W_1,\dots,W_t)}{nd^{t-1}}=\| P_tMP_{t-1} M\cdots P_2MP_1u\|_1.
\end{align}
Let
$
    v_1= P_1u=\frac{1}{n}\1_{W_1}, 
    v_{i+1} =P_{i+1}Mv_{i}.
$
Then 
\begin{align}\frac{e(W_1,\dots,W_t)}{nd^{t-1}}=\|v_{t}\|_1, \quad 
    \|v_1\|_1=\alpha_1, \quad \|v_1\|_2=\frac{\sqrt{\alpha_1}}{\sqrt n}.
\end{align}

Decompose $v_i=x_i+y_i$, where $x_i$ is the part of $v$ that is a scalar multiple of $u$, and $y_i=Pv_i$ is  orthogonal to $u$.
We first prove the following lemma. 
The argument is based on the proof of the expander mixing lemma \eqref{eq:expandermixing} for $d$-regular graphs (see e.g. \cite{chungspectral}).
\begin{lemma}\label{lem:y}
For $1\leq i\leq t-1$,
\begin{align*} 
  \left | \|v_{i+1}\|_1-\alpha_i \|v_i\|_1  \right|\leq \frac{\lambda}{d}\sqrt{n\alpha_{i+1}(1-\alpha_{i+1})}\|y_i\|_2.
\end{align*}
\end{lemma}
\begin{proof}
Since the entries of $v_i$ represent probabilities, each $v_i$ is entrywise nonnegative. Denote $\1=(1,\dots,1)^{\top}$. We have
\begin{align*}
    \|v_{i+1}\|_1&=\|P_{i+1}Mv_i\|_1=\1_{W_{i+1}}^{\top} Mv_{i}\\
    &=\alpha_{i+1}\1^{\top} M \left(\frac{\|v_{i}\|_1}{n} \1\right)+ \left(\1_{W_{i+1}}-\alpha_{i+1}\1\right)^{\top} M \left(v_i-\frac{\|v_{i}\|_1}{n} \1\right)\\
    &\quad +(\1_{W_{i+1}}-\alpha_{i+1}\1)^{\top} M \frac{\|v_{i}\|_1}{n} \1+\alpha_{i+1}\1^{\top} M \left(v_i-\frac{\|v_{i}\|_1}{n} \1\right).
\end{align*}
Note that, since $\1$ is the  eigenvector of $M$ and $M^T$ with eigenvalue $1$,
\begin{align*}
   & (\1_{W_{i+1}}-\alpha_{i+1}\1)^{\top} M \frac{\|v_{i}\|_1}{n} \1=0,\quad \alpha_{i+1}\1^{\top} M \left(v_i-\frac{\|v_{i}\|_1}{n} \1\right)=0.
\end{align*}
We obtain 
\[  \|v_{i+1}\|_1
=\alpha_{i+1} \|v_i\|_1+ \left(\1_{W_{i+1}}-\alpha_{i+1}\1\right)^{\top} M \left(v_i-\frac{\|v_{i}\|_1}{n} \1\right).
\]
Therefore by Cauchy inequality,  $\left|  \|v_{i+1}\|_1-\alpha_{i+1} \|v_i\|_1\right|$ can be bounded by 
\begin{align*}
   \left| \left(\1_{W_{i+1}}-\alpha_{i+1}\1^{\top}\right) M \left(v_i-\frac{\|v_{i}\|_1}{n} \1\right)\right|
  &\leq \frac{\lambda}{d} \|\1_{W_{i+1}}-\alpha_{i+1}\1^{\top}\|_2 \left\|v_i-\frac{\|v_{i}\|_1}{n} \1 \right\|_2 \\
   &=\frac{\lambda}{d}\sqrt{n\alpha_{i+1}(1-\alpha_{i+1}) }\|y_i\|_2.
\end{align*}
\end{proof}
With Lemma \ref{lem:y}, we finish the proof of Theorem \ref{thm:mixing}.
\begin{proof}[Proof of Theorem \ref{thm:mixing}]
We would like to control $\| y_i \|_2$.
Note that 
\begin{align*}
    \|y_1\|_2= \left\|v_1-\frac{\|v_{1}\|_1}{n} \1 \right\|_2=\frac{1}{n}\left\|\1_{W_1}-\frac{|W_1|}{n}\1 \right\|_2= \sqrt{\frac{\alpha_1(1-\alpha_1)}{n}}.
\end{align*}
For $i\geq 2$, $
    \|y_i\|_2\leq \|v_i\|_2=\|P_iMv_{i-1}\|_2\leq \|v_{i-1}\|_2.
$
We obtain $\|y_i\|_2 \leq \|v_1\|_2=\sqrt{\alpha_1/n}.$
From Lemma \ref{lem:y}, we  have
\begin{align}
   \left| \|v_{t}\|_1 -\alpha_t \|v_{t-1}\|_1  \right|
    &\leq \frac{\lambda}{d} \sqrt{\alpha_1\alpha_{t}(1-\alpha_{t})}, \notag \\ 
      \left|  \alpha_t\|v_{t-1}\|_1 -\alpha_{t-1}\alpha_t \|v_{t-2}\|_1   \right| 
    &\leq \frac{\lambda}{d} \alpha_t\sqrt{\alpha_1\alpha_{t-1}(1-\alpha_{t-1})}, \notag \\
     &\vdots \notag\\
    \left| \alpha_4 \cdots \alpha_t \|v_3\|_1 - \alpha_3 \cdots \alpha_t \|v_2\|_1 \right|
     &\leq \frac{\lambda}{d}(\alpha_4\cdots\alpha_t) \sqrt{\alpha_1 \alpha_3 (1-\alpha_3)},  \notag\\
\left|\alpha_3\cdots\alpha_t\|v_2\|_1-\alpha_2\cdots\alpha_t\|v_1\|_1 \right|
&\leq \frac{\lambda}{d}(\alpha_3\cdots\alpha_t) \sqrt{\alpha_1(1-\alpha_1)\alpha_2(1-\alpha_2)}. \label{eq:telescope}
\end{align}
Since $\|v_1\|_1=\alpha_1$, for $t\geq 2$, combining the inequalities above and 
applying the triangle inequality leads to
\begin{align} \label{eq:vt}
  & \left|\|v_{t}\|_1 - \alpha_1 \cdots \alpha_t\right| \\
  &\quad \leq \frac{\lambda}{d} \left(\alpha_3\cdots\alpha_t \sqrt{\alpha_1(1-\alpha_1)\alpha_2(1-\alpha_2)}+ 
  \sum_{i=3}^{t} 
  \sqrt{\alpha_1 \alpha_i(1-\alpha_i)}
  \prod_{k=i+1}^{t}\alpha_k \right). \notag
\end{align}
In the $i=t$ term, the empty product is 
defined to equal 1.
Using the inequality $\sqrt{x(1-x)}\leq \frac{1}{2}$ for any $x\in [0,1]$
and that $\alpha_i \in [0,1]$, \eqref{eq:vt} implies
\begin{align}
    \left| \frac{e(W_1,\dots,W_t)}{n d^{t-1}} - \alpha_1\cdots\alpha_t \right|\leq \frac{\lambda}{d}\left( \frac{1}{4}+\frac{t-2}{2}\right).
\end{align}
 This completes the proof.
\end{proof}

\subsection{Quasi-regular hypergraph expanders}\label{sec:bipartiteextension}

The  quasi-regular hypergraph constructed in Section \ref{sec:quasiregular} has a similar expansion property as follows.
\begin{theorem} 
Let $H$ be the hypergraph constructed in Definition \ref{eq:def_quasi}.
Let $W_i\subseteq V_i$,  $1\leq i\leq t$ be any non-empty subsets. Denote 
$
\alpha_i:=\frac{|W_i|}{n_i}\in [0,1],  1\leq i\leq t,
$
and let
$
e(W_1,\dots, W_t)
$
be the number of hyperedges between $W_1,\dots, W_t$. Then the following expansion property holds:
\label{thm:mixingbipartite}
\begin{align*}
&\left| \frac{e(W_1, \ldots, W_t)}{n_1d_1\cdots d_{2t-3}} - \prod_{i=1}^t \alpha_i \right|\\
\leq &\frac{\lambda^{(1)}}{\sqrt{d_1d_2}}\sqrt{\alpha_1(1-\alpha_1)\alpha_2(1-\alpha_2)}\prod_{k=3}^t\alpha_k+\sum_{k=2}^{t-1}\frac{\lambda^{(k)}}{\sqrt{d_{2k-1}d_{2k}}}\sqrt{\alpha_1\alpha_{k+1}(1-\alpha_{k+1})}\prod_{i=k+2}^t \alpha_i\\
  \leq & \frac{\lambda^{(1)}}{4\sqrt{d_1d_2}}+\sum_{k=2}^{t-1}\frac{\lambda^{(k)}}{2\sqrt{d_{2k-1}d_{2k}}}.
\end{align*}
\end{theorem}
When $t=2$, Theorem \ref{thm:mixingbipartite} reduces to  the expander mixing lemma for bipartite biregular graphs  \cite{haemers1995interlacing,de2012large}. We present the proof of Theorem \ref{thm:mixingbipartite}  in Appendix ~\ref{sec:appendix_mixbipartite}.



\section{Tensor complexity}\label{sec:property}

In order to complete a partially observed matrix or tensor,
some kind of prior knowledge of its structure is required.
The tensor that is output by the learning algorithm 
will then be the least complex one that is consistent with 
the observations.
Consistency may be defined as either exactly matching the observed
entries---in the case of zero noise---or 
being close to them under some loss---in 
the case where the observations are corrupted.
We now argue for the use of the tensor 
max-quasinorm (see Definition~\ref{def:max-qnorm} below) as a measure of complexity.
Towards this aim, we also show a number of previously unknown
properties about the max-quasinorm.

For matrices, 
the most common measure of complexity is the rank.
In the tensor setting, 
there are various definitions of rank 
\citep{kolda2009tensor}.
However, in this paper 
we will work with the rank defined via the
CP decomposition as
\begin{equation}
    \mathrm{rank}(T) = 
    \min
    \left\{ r \; \Big| \; T = \sum_{i=1}^r u^{(1)}_i \circ \cdots \circ u^{(t)}_i
    \right\} ,
    \label{eq:tensor_rank}
\end{equation}
where each  vectors $u_i^{(j)} \in \R^n$.
Note that the decomposition above is atomic and equivalent
to the decomposition used to define matrix rank when $t=2$.
The sum is composed of $r$
rank-1 tensors expressed as the outer products 
$u_i^{(1)} \circ \cdots \circ u_i^{(t)}$.
Our analysis uses Kronecker and Hadamard
products of tensors. 
These are generalizations of the usual
definitions for matrices
in the obvious way.
\begin{definition}
    Let $T \in \bigotimes_{i=1}^t \R^{n_i}$
    and $S \in \bigotimes_{i=1}^t \R^{m_i}$.
    We define the {\bf Kronecker product}
    of two tensors
    $(T \otimes S) \in \bigotimes_{i=1}^t \R^{n_i m_i}$
    as the tensor with entries
    \[
    (T \otimes S)_{k_1, \ldots, k_t}
    =
    T_{i_1, \ldots, i_t}
    S_{j_1, \ldots, j_t}
    \quad \text{for}\quad 
    k_1 = j_1 + m_1 (i_1 - 1) , \ldots, 
    k_t = j_t + m_t (i_t - 1).\]
\end{definition}

\begin{definition}
    Let $T \in \bigotimes_{i=1}^t \R^{n_i}$
    and $S \in \bigotimes_{i=1}^t \R^{n_i}$.
    We define the {\bf Hadamard product}
    of two tensors
    $(T *S) \in \bigotimes_{i=1}^t \R^{n_i}$
    as the tensor with indices
    $
    (T *S)_{i_1, \ldots, i_t}
    =
    T_{i_1, \ldots, i_t}
    S_{i_1, \ldots, i_t} 
    $.
\end{definition}

\subsection{Max-qnorm}

The max-norm of a matrix 
(also called $\gamma_2$-norm)
is a common relaxation of rank.
It was originally proposed in the theory of Banach spaces
\citep{tomczak-jaegermann1989},
but has found applications in communication complexity 
\citep{linial2007,lee2008,matousek2014}
and matrix completion
\citep{srebro2005rank,rennie2005fast,foygel2011,heiman2014deterministic,cai2016,foucart2017}.
For a matrix $A$, the max-norm of $A$ is defined as
\[
\| A \|_{\max} := 
\min_{U,V: A = U V^{\top}} \| U \|_{2,\infty} \| V \|_{2,\infty} .
\]
We can generalize its definition to tensors,
following \cite{ghadermarzy2018},
with the caveat that it then becomes a quasinorm
since the triangle inequality is not satisfied. 
\begin{definition}
\label{def:max-qnorm}
    Define the 
    {\bf max-quasinorm} (or {\bf max-qnorm})
    of an order-$t$ tensor $T\in \bigotimes_{i=1}^t \mathbb R^{n_i}$
    as
  \[
  \| T \|_{\max} = 
  \min_{T = U^{(1)} \circ \cdots \circ U^{(t)}}
  \prod_{i=1}^t \| U^{(i)} \|_{2,\infty}
  ,
  \quad  \text{where}\quad  \| U \|_{2,\infty} = \max_{\|x\|_2=1} \| U x\|_\infty,\]
  i.e.\
  the maximum $\ell_2$-norm
  of any row of $U$,
  and each of the $U^{(i)} \in \R^{n_i \times r}$
  for some $r$.
\end{definition}
 
 The following lemma provides some basic properties of the max-quasinorm for tensors.
\begin{lemma}[\cite{ghadermarzy2018}, Theorem 4]
\label{lem:max-qnorm}
  Let $t \geq 2$,
  then any two order-$t$ tensors $T$ and $S$
  of the same shape satisfy the following properties:
  \begin{enumerate}
  \item $ \|T \|_\mathrm{max} = 0$ if and only if ~$T = 0$.
      \item $\| c T \|_\mathrm{max} = |c| \| T \|_\mathrm{max}$, where $c \in \R$.
      \item $
        \| T + S \|_\mathrm{max} 
        \leq 
        \left( \| T \|_\mathrm{max}^{2/t} + \|S\|_\mathrm{max}^{2/t} \right)^{t/2}
        \leq
        2^{t/2 - 1}
        \left( \|T\|_\mathrm{max} + \|S\|_\mathrm{max} \right)$.
  \end{enumerate}
\end{lemma}

Note that property (3)
in Lemma~\ref{lem:max-qnorm}
implies that $\| \cdot \|_\mathrm{max}$
is a so-called 
$p$-norm
with $p=2/t$
and also a quasinorm with constant $2^{t/2-1}$
\citep{dilworth1985}.
Finally, in the matrix case it is a norm. 
As a matrix norm,
many properties and equivalent definitions of the max-norm are known,
and it can be computed via semidefinite programming
\citep{linial2007,lee2008,matousek2014}.
In the tensor case, much less is known about the max-qnorm.
We now prove generalizations of some of these
properties that hold for tensors
(proof in Section \ref{pf:max-qnorm-properties}).
\begin{theorem}
\label{thm:max-qnorm-properties}
    Let $T \in \bigotimes_{i=1}^t \R^{n_i}$
    and $S \in \bigotimes_{i=1}^t \R^{m_i}$.
The following max-qnorm properties hold:
\begin{enumerate}
    \item $ \| T_{I_1, \ldots, I_t} \|_\mathrm{max}
            \leq
            \| T \|_\mathrm{max}$
    for any subsets of indices
    $I_i \subseteq [n_i]$.
    \item $\| T \otimes S \|_\mathrm{max}
            \leq
            \| T \|_\mathrm{max} \| S \|_\mathrm{max}$. 
    \item $\| T *S \|_\mathrm{max} \leq \| T \otimes S \|_\mathrm{max}$,
    where $T,S\in  \bigotimes_{i=1}^t \R^{n_i}$.
\item $
\| T *T  \|_\mathrm{max} \leq \| T \|_\mathrm{max}^2.$
\end{enumerate}
\end{theorem}

For any matrix $A$,
there is a surprising relationship between 
max-norm and rank:
\begin{equation}
\label{eq:max-rank-matrix}
| A |_{\infty}
\leq
\| A \|_\mathrm{max} 
\leq 
\sqrt{\mathrm{rank}(A)} \cdot | A |_{\infty},
\end{equation}
which does not depend on the size of $A$.
We denote the entry-wise infinity-norm
$|A |_{\infty} = \max_{i,j} | A_{ij} |$,
and similarly
$|T|_\infty = \max_{i_1, \ldots, i_t} |T_{i_1,\ldots, i_t}|$
for tensors.
The proof of
\eqref{eq:max-rank-matrix}
is a result of John's theorem,
and is given in \cite{linial2007}.
For the tensor generalization, 
we prove the following theorem.

\begin{theorem}
\label{thm:max-rank}
Let $T \in \bigotimes_{i=1}^t \R^n$
with $\mathrm{rank}(T) = r$.
Then we have that
\[
|T|_\infty
\leq \| T \|_\mathrm{max} 
\leq 
\sqrt{r^{t^2-t-1}} 
\cdot
|T|_\infty .
\]
\end{theorem}
The proof of Theorem \ref{thm:max-rank} is included in Section \ref{sec:proof_max_rank_property}.
This improves  Theorem 7(ii) in \cite{ghadermarzy2018} by a factor of $\sqrt{r}$. If we take $t=2$ we get that
$\| A \|_\mathrm{max} \leq \sqrt{\mathrm{rank}(A)} \cdot 
| A |_{\infty}$,
which is the same as  \eqref{eq:max-rank-matrix}.
It remains an open question whether a better bound exists
for all $t \geq 2$  in terms of the dependence on $t$.
The numerical experiments of
\cite{ghadermarzy2018},
which used a bisection method to estimate
the max-qnorm of tensors of known rank,
suggest that an improvement is possible.
They study tensors formed from random factors,
finding that increasing $t$ by one leads
to approximately $\sqrt{r}$ increase.
This suggests the conjecture that perhaps
$
\|T \|_\mathrm{max} 
\leq 
\sqrt{r^{t-1}} 
\cdot |T|_\infty
$
is the optimal bound.
Further support for this scaling with $r$ and $t$
comes from incoherent tensors.
Our definition is inspired by 
but slightly different from that used by
\citet{barak2016noisy}
but reduces to the usual matrix incoherence condition.

\begin{definition}
A rank $r$ tensor $T$ is said to be {\bf $C$-incoherent}
if there exists a rank $r$ factorization
$T = U^{(1)} \circ \cdots \circ U^{(t)}$
such that
$|U^{(i)}|_\infty \leq C$
for $i \in [t]$.
\end{definition}

\begin{prop}\label{prop:incoherent}
Let $T$ be a $C$-incoherent tensor.
Then $\|T \|_\mathrm{max} \leq C^t \sqrt{r^t}$,
and $|T|_\infty \leq C^t$.
\end{prop}

The proof of Proposition \ref{prop:incoherent} is included in Section \ref{sec:proof_of_incoherent}.
In any case, 
Theorem~\ref{thm:max-rank} is still useful
for low rank tensor completion,
as it implies that an upper bound 
on the generalization error in terms of the max-qnorm
can be translated into a bound that 
depends on the rank.
That upper bound does not depend on $n$, 
which is crucial for attaining sample complexity linear in $n$.

We have found an  improved lower bound on $\|T\|_{\max}$ via tensor matricization, sometimes called tensor unfolding or tensor flattening. The proof is in Section~\ref{pf:improved}.

\begin{definition}[\citet{kolda2009tensor}]
Let $T \in \bigotimes_{i=1}^t \R^{n_i}$. 
For $1\leq i\leq t$, the mode-$i$ \textbf{matricization} of $T$ 
is a matrix denoted by   
$T_{[i]} \in \R^{n_i} \times \R^{\prod_{j \neq i} n_j}$ 
such that for any index  $(j_1,\dots, j_t)$,
\[ \left(T_{[i]}\right)_{j_i,k}=T_{j_1,\dots, j_t},\]  
with 
$k=1+\sum_{s=1,s\not=i}^t (j_s-1)N_s$ 
and
$N_s=\prod_{m=1,m\not=i}^{s-1} n_m.$
\end{definition}

\begin{prop}\label{thm:improved}
Let $T \in \bigotimes_{i=1}^t \R^{n_i}$.
Then 
$\displaystyle
\| T \|_\mathrm{max} 
\geq 
\max_{1\leq i\leq t}\| T_{[i]} \|_\mathrm{max}
\geq 
\max_{i_1, \ldots, i_t} |T_{i_1,\ldots, i_t}|. 
$
\end{prop}

\subsection{Sign tensors}

In order to connect expansion properties of the hypergraph $H$
to the error of our proposed tensor completion algorithm,
we will work with sign tensors. 
A sign tensor $S$ has all entries equal to $+1$ or $-1$, i.e.
$S \in \bigotimes_{i=1}^t \{ \pm 1\}^{n_i}$.
The sign rank of a sign tensor $S$ is defined as
\begin{equation}
    \mathrm{rank}_\pm (S) = 
    \inf
    \left\{ r \; \Big| \; 
    S = \sum_{i=1}^r s^{(1)}_i \circ \cdots \circ s^{(t)}_i,
    \;
    s_i^{(j)} \in \{ \pm 1 \}^{n_i}
    \mbox{ for $i \in [r]$ and $j \in [t]$}
    \right\}.
    \label{eq:sign_rank}
\end{equation}
Using rank-1 sign tensors as our atoms,
we can construct a nuclear norm \citep{ghadermarzy2018}:
\begin{definition}
We define the {\bf sign nuclear norm} for a tensor $T$ as
\begin{equation}
    \| T \|_\pm = 
    \inf \left\{ \sum_{i=1}^r |\alpha_i|
    \; \Big| \; T = \sum_{i=1}^r \alpha_i S_i
    \mbox{ where }
    \alpha_i \in \R, \,
    \mathrm{rank}_\pm (S_i) = 1 
    \right\} .
\end{equation}
\end{definition}
Note that the set of all rank-1 sign tensors forms a basis
for $\bigotimes_{i=1}^t \R^n$, 
so this decomposition into rank-1 sign tensors is always possible;
furthermore this is a norm for tensors and matrices
\citep{ghadermarzy2018,heiman2014deterministic}.
The sign nuclear norm is called
the ``atomic M-norm'' by \cite{ghadermarzy2018}
and the ``atomic norm'' by \cite{heiman2014deterministic}.

The next lemma which relates $\| \cdot \|_\pm$ and $\| \cdot \|_\mathrm{max}$ 
follows from a multilinear generalization of Grothendieck's inequality
(the proof is given  in Section~\ref{pf:grothendieck}).
We use $K_G$ to denote Grothendieck's constant over the reals.
For detailed background, see \cite{tomczak-jaegermann1989}.

\begin{lemma}
\label{lem:grothendieck}
The sign nuclear norm and max-qnorm satisfy
  $
  \| T \|_\pm \leq K_G^{t-1} \| T \|_\mathrm{max},
  $
  where $K_G$ is Grothendieck's constant over $\mathbb R$.
\end{lemma}


\section{Tensor completion}\label{sec:completion}

\subsection{Proof of Theorem \ref{thm:main3}}

Consider a rank-1 sign tensor
$S = s_1 \circ \cdots \circ s_t$
with $s_j \in \{\pm 1\}^n$.
Let $J$ be the tensor of all ones and $S' = \frac{1}{2} (S + J)$,
so that $S'$ is shifted to be a tensor of zeros and ones. Then 
\begin{align}\label{eq:defS}
  S'_{i_1,\dots, i_t}=\begin{cases}
    1 & \text{if }  (s_1)_{i_1}\cdots (s_t)_{i_t}=1,\\
    0 & \text{if }   (s_1)_{i_1}\cdots (s_t)_{i_t}=-1.
  \end{cases}
\end{align}
Define the sets
\begin{align} \label{eq:defW}
  W_j := \{i \in [n] : (s_j)_i =- 1 \}.
\end{align}
Let  $\mathcal{S}_t$ is the set of even $t$-strings in $\{0,1\}^t$.
An even string has an even number of 1's in it, 
e.g.\ for $t=3$ we have 000, 110, 101, 011 as even strings. 
The number of these strings is 
$
| \mathcal{S}_t | =  2^{t-1}
$, 
so we can  enumerate all possible even $t$-strings from $1$ to $2^{t-1}$, 
denoted by $w_1,\dots, w_{2^{t-1}}\in \{0,1\}^t$. 
Now for all $1\leq i\leq t$
and $1\leq j\leq 2^{t-1}$, 
we define the sets $W_{i,j}$ by
\begin{align}
  W_{i,j}=\begin{cases}
    W_i, & \text{if } (w_j)_i=1, \\ 
    [n]\setminus W_i, & \text{if } (w_j)_i=0.
  \end{cases}
\end{align}
By considering the sign of entries in the components
of $S$,
we have the following decomposition for $S'$ as a sum of rank-$1$ tensors
(a derivation of \eqref{eq:sumsign} is provided in Appendix~\ref{pf:sumsign}):
\begin{align}\label{eq:sumsign}
  S' = \sum_{j=1}^{2^{t-1}} 
  1_{W_{1,j}}\circ \cdots \circ 1_{W_{t,j}} .
\end{align}
Now we consider the deviation in the sample mean from the 
mean over all entries in $S$:
  \begin{align*}
    \left| 
    \frac{1}{n^t} \sum_{e \in [n]^t} 
    S_{e} 
    - 
    \frac{1}{|E|} \sum_{e \in E} 
    S_{e} 
    \right|
    &= 
      \left| \frac{1}{n^t} \sum_{e\in [n]^t} (2 S'_e - 1) - 
      \frac{1}{|E|} \sum_{e \in E} (2 S'_e - 1) \right| \\
      &= 2 \left| \frac{1}{n^t} \sum_{e\in [n]^t} S'_e - 
      \frac{1}{|E|} \sum_{e \in E} S'_e \right| \\
    &= 2 
    \left| 
      \frac{\sum_{j=1}^{2^{t-1}} \prod_{i=1}^t |W_{i,j}| }{n^t} - 
      \frac{\sum_{j=1}^{2^{t-1}} e(W_{1,j}, \ldots, W_{t,j}) }{|E|}
    \right|\\
    &\leq 2 
      \sum_{j=1}^{2^{t-1}}
      \left| \frac{|W_{1,j}| \cdots |W_{t,j}|}{n^t} - 
      \frac{ e(W_{1,j}, \ldots, W_{t,j}) }{|E|} \right| .
  \end{align*}
 Applying \eqref{eq:Hmixinglemma} in Lemma \ref{lem:tensormixing} to the sets $W_{1,j}\subseteq V_1,\dots, W_{t,j}\subseteq V_t$ for each $1\leq j\leq 2^{t-1}$, we get that
  \begin{align}\label{eq:signbound1}
    \left| \frac{1}{n^t} \sum_{e \in [n]^t} S_e 
    - \frac{1}{|E|} \sum_{e \in E} S_e \right| &\leq 2 
      \sum_{j=1}^{2^{t-1}} \frac{\lambda_2(H)}{|E|} \sqrt{|W_{1,j}|\cdots |W_{t,j}|}  \notag\\
    &= 2\sum_{j=1}^{2^{t-1}} \frac{\lambda_2(H)n^{t/2}}{|E|}\sqrt{\frac{|W_{1,j}|\cdots |W_{t,j}|}{n^t}} \leq \frac{2^t n^{t/2} \lambda_2(H)}{|E|}. 
    \end{align}

  We now write the tensor 
  $T = \sum_i \alpha_i S_i$ 
  as a sum of rank-1 sign tensors $S_i$, with coefficients
  $\alpha_i \in \mathbb{R}$. 
  Let $\| \cdot \|_{\pm}$ be the tensor sign nuclear norm,
  i.e.\ $\|T \|_\pm = \sum_i |\alpha_i|$.
  Then  for a general tensor $T$, we can apply \eqref{eq:signbound1} to each $S_i$, and by triangle inequality we get
  \[
  \left| \frac{1}{n^t} \sum_{e \in [n]^t} T_e 
    - \frac{1}{|E|} \sum_{e \in E} T_e \right|
  \leq \frac{2^tn^{t/2}\lambda_2(H)}{|E|}
  \; \| T \|_\pm .
  \]
  This holds for any tensor $T$. Now we apply this inequality to
  the tensor of squared residuals
  $R:= (\hat{T} - T ) * (\hat{T} - T).$
  Since we solve for $\hat{T}$
  with equality constraints, we have that
  $R_e = 0$ for all $e \in E$.
  Thus,
  \begin{align}
  &\frac{1}{n^t}\|\hat{T}-T\|_F^2= 
  \left| \frac{1}{n^{t}} \sum_{e \in [n]^t} R_e \right| \notag\\
  \leq &  \frac{2^tn^{t/2}\lambda_2(H)}{|E|}
  \|R\|_\pm \leq  
 \frac{2^tn^{t/2}\lambda_2(H)}{|E|}
  K_G^{t-1} \;
  \|R\|_\textrm{max} & \mbox{(Lemma~\ref{lem:grothendieck})}
  \notag\\
  \leq  & 
  \frac{2^tn^{t/2}\lambda_2(H)}{|E|}
  K_G^{t-1} \;
  \|\hat{T} - T\|_\textrm{max}^2 
  & \mbox{(Theorem \ref{thm:max-qnorm-properties}, part 4)} \notag \\
  \leq  &
 \frac{2^{2t-2}n^{t/2}\lambda_2(H)}{|E|}
  K_G^{t-1}
  \left(\|\hat{T} \|_\mathrm{max} + \|T \|_\mathrm{max}\right)^2 .
  & \mbox{(Lemma~\ref{lem:max-qnorm})}
  \label{eq:maxinequality}
  \end{align}
  Since $\hat{T}$ is the output of our optimization routine
  and $T$ is feasible, 
  $\| \hat{T} \|_\mathrm{max} \leq \| T \|_\mathrm{max}$.
  This leads to the final result.

\subsection{Proof of Theorem \ref{thm:main}}
Consider a hypercubic tensor $T$ of order-$t$, i.e.\ 
$T \in \R^{n\times \cdots \times n} $.
We sample the entry $T_e$ whenever 
$e = (i_1,\dots, i_t)$ is a hyperedge in $H$ defined in Section \ref{sec:regular}.  
Then the sample size is $|E|=nd^{t-1}$.  
Consider a rank-1 sign tensor
  $S = s_1 \circ \cdots \circ s_t$
  with $s_j \in \{\pm 1\}^n$. Following the same steps in the proof of Theorem \ref{thm:main3}, we obtain
  \begin{align*}
    \left| 
    \frac{1}{n^t} \sum_{e \in [n]^t} 
    S_{e} 
    - 
    \frac{1}{|E|} \sum_{e \in E} 
    S_{e} 
    \right|
    &\leq 2 
      \sum_{j=1}^{2^{t-1}}
      \left| \frac{|W_{1,j}| \cdots |W_{t,j}|}{n^t} - 
      \frac{ e(W_{1,j}, \ldots, W_{t,j}) }{nd^{t-1}} \right| .
  \end{align*}
  Applying Theorem~\ref{thm:mixing} to the sets $W_{1,j}\subseteq V_1,\dots, W_{t,j}\subseteq V_t$ for each $1\leq j\leq 2^{t-1}$, we get that
  \begin{align}\label{eq:signbound}
    \left| \frac{1}{n^t} \sum_{e \in [n]^t} S_e 
    - \frac{1}{nd^{t-1}} \sum_{e \in E} S_e \right|
    &\leq 
      2\sum_{j=1}^{2^{t-1}} \frac{(2t-3)\lambda}{4d}=2^{t-2}(2t-3)\frac{\lambda}{d}.
    \end{align}

  We now write the tensor 
  $T = \sum_i \alpha_i S_i$ 
  as a sum of rank-1 sign tensors $S_i$, with coefficients
  $\alpha_i \in \mathbb{R}$. 
Define $R=(\hat{T}-T)*(\hat{T}-T)$.
 Following the same steps in the proof of Theorem \ref{thm:main3}, we have
  \begin{align*}
 \frac{1}{n^t}\|\hat{T}-T\|_F^2= \left| \frac{1}{n^{t}} \sum_{e \in [n]^t} R_e \right| 
  \leq 2^{t-2}(2t-3)\frac{\lambda}{d}
  K_G^{t-1}
  \left(\|\hat{T} \|_\mathrm{max} + \|T \|_\mathrm{max}\right)^2 .
  \end{align*}
  Since 
  $\| \hat{T} \|_\mathrm{max} \leq \| T \|_\mathrm{max}$,
  this leads to the final result with a constant  
  $C_t=2^{t} (2t-3)K_G^{t-1}$.

\subsection{Tensor completion with erroneous observations}


Now we turn to the case when our observations $Z$ 
of the original tensor $T$
are corrupted by errors $\nu$.
We will call this noise, but it can be anything, even chosen adversarially,
so long as it's bounded.
Let 
$Z\in \mathbb R^n\times  \cdots \times \mathbb R^n$ 
be the tensor we observe with 
$Z_{e} = 0$ 
if 
$e \not \in E$
and
$Z_e = T_e + \nu_e$
for 
$e \in E$. 
In this case, we study the solution to the following optimization problem:
\begin{align}
\label{eq:noisy_problem}
     \min_{X}\quad & \|X\|_{\max},   \\
    \text{subject to} 
    \quad  
    &
    \frac{1}{|E|}
    \sum_{ e \in E} 
    (X_e - Z_e)^2
    \leq \delta^2, \notag
\end{align}
for some $\delta>0$. 
The parameter $\delta$ is a bound on the root mean squared
error of the observations.
In a probabilistic setting,
we may pick this parameter so that the constraint
holds with sufficiently high probability.
We obtain the following corollary of Theorem \ref{thm:main}, 
with the proof in Appendix~\ref{pf:noisycompletion}.
\begin{corollary}\label{cor:noisycompletion} 
Let $E$ be the hyperedge set of $H$ defined in Section \ref{sec:construction}. Suppose we observe $Z_e = T_e + \nu_e$ 
for all $e \in E$ with bounded mean square error satisfying
\begin{align}
\label{cor:noisebound}
 \frac{1}{|E|}\sum_{e \in E} \nu_e^2\leq \delta^2.  
\end{align}
Then solving the optimization problem \eqref{eq:noisy_problem} will give us a solution $\hat{T}$ that satisfies
\begin{align}\label{eq:inrreducible_delta}
    \frac{1}{n^t}\|\hat{T}-T \|_F^2\leq 2^{t} (2t-3)K_G^{t-1}\| T \|_\mathrm{max}^2 \frac{\lambda}{d} +4\delta^2.
\end{align}
\end{corollary}

\begin{remark}
\eqref{eq:inrreducible_delta} 
has an error with an irreducible term $O(\delta^2)$,
in contrast to results such as \cite{ghadermarzy2018}
where the error goes to zero even in the presence of noise.
In our case, this term cannot be overcome because the
errors do not have any of the nice properties of noise.
For example, all observations $Z$ 
could be shifted by an amount $\delta$,
which would bias the estimation of $T$ to $T + \delta J$.
\end{remark}

\section{Practical algorithm for finding minimum complexity tensors}
\label{sec:optimization}

In such tensor factorization problems, one usually
picks a rank $r$ and alternately
minimizes the objective function over the factors
$U^{(i)}$, making the
overall approach coordinate descent.
Note that the optimization over a single factor is
in fact a convex problem for which there exist polynomial time algorithms.
However, care must be taken in designing these
convex subproblems for efficiency.
We will relax the problem in a way to make it amenable
to proximal gradient descent methods.
This will require us to compute the proximal operator 
of $\|\cdot\|_{2,\infty}$
and the gradient of the smooth part of the 
modified objective.

\subsection{Relaxed algorithm}

We propose the following more practical relaxation of \eqref{eq:noisy_problem}.
Rather than deal with the constrained problem,
we instead optimize:
\begin{align}
\label{eq:relaxed}
    \min_{X, R} \quad
    & \|X \|_\mathrm{max}
    + \frac{\kappa}{2} \|\mask * (X - Z - R)\|_F^2
    + \frac{\beta}{2} \| \mask * R \|_F^2
    \\
    \text{subject to} 
    \quad  
    &
    \| \mask * R \|_F
    \leq \delta . \notag
\end{align}
The mask tensor $\mask$ has $M_e = 1$ if $e \in E$ and $M_e = 0$ otherwise.
This makes $\mask$ equivalent to the adjacency operator of the observation
hypergraph $H$, 
so the constraint is the same as in the original noisy problem \eqref{eq:noisy_problem}.
However, we have relaxed the problem by
introducing the auxiliary variable $R$.
Here we absorb $\sqrt{|E|}$ 
into the noise magnitude $\delta$,
to avoid clutter.
The relaxation parameter $\kappa$ is taken to be large
which keeps $R \approx X-Z$ 
within the observation mask;
outside the mask, we may assume $R=0$.
The parameter $\beta$ is a small
smoothing of the hard constraint
taken for technical reasons laid out in Appendix~\ref{sec:supp_alg}.
For all experiments, we take $\kappa = 100$ and $\beta = 1$.


The ability to compute gradients and evaluate
the prox of the $\ell_{2,\infty}$ norm are
all the ingredients needed to implement coordinate descent
for \eqref{eq:relaxed}.
The details of how these are computed
and a variable projection step to minimize out $R$ 
are provided in  Appendix~\ref{sec:supp_alg}.
At each coordinate descent step,
we apply the accelerated proximal gradient method to the cost.
The cost in terms of the coordinate blocks 
$\{ U^{(i)} \}_{i=1}^t$
is written as:
\begin{equation}
C = 
\frac{1}{2} \left(\kappa (1-\mu)^2 + \mu^2 \beta \right)
    \|\mask * (U^{(1)} \circ \cdots \circ U^{(t)} - Z)\|_F^2 
    +
    \prod_{i=1}^t \| U^{(i)} \|_{2,\infty},
    \label{eq:cost_mod}
\end{equation}
with $\mu = \mu(U^{(1)} \circ \cdots \circ U^{(t)})$ given by 
\eqref{eq:mu_defn}.
After the first iteration of coordinate descent,
we find that it is important to rescale the factors
to have equal norms across all columns.
This step does not affect the goodness of fit, 
but it can cause the max-qnorm penalty term to increase.
However, we have found it makes the algorithm more stable.
Not shown in \eqref{eq:cost_mod}, 
we also include a small squared Frobenius penalty 
$0.01 \cdot \|U\|_F^2$ on each factor for numerical stability.
Python code to implement this method and reproduce our experiments 
is available from
{\url{https://github.com/kharris/max-qnorm-tensor-completion}}.

\subsection{Numerical experiments}
\label{sec:experiments}

\begin{figure}[t]
    \centering
    \includegraphics[width=.8\linewidth,trim={20 0 20 0},
    clip]{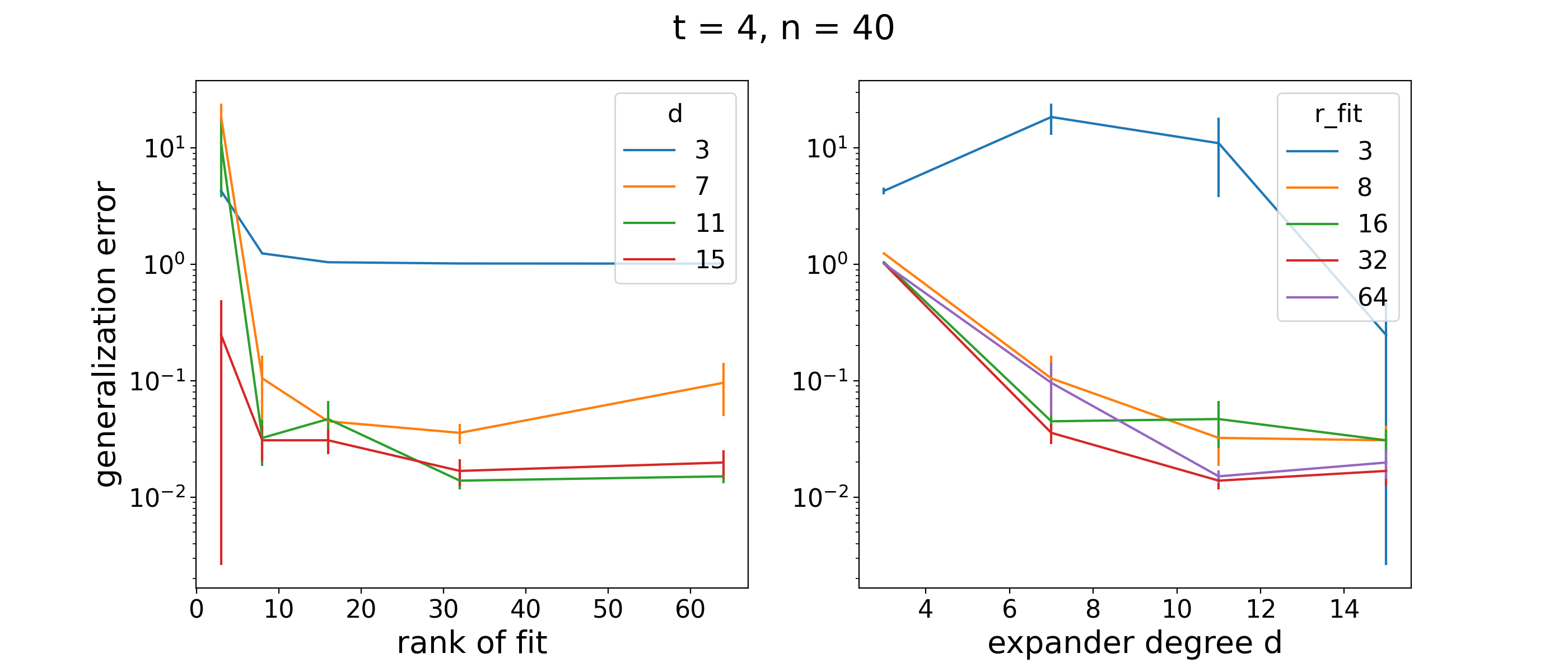}
    
    \includegraphics[width=.8\linewidth,trim={20 0 20 0},
    clip]{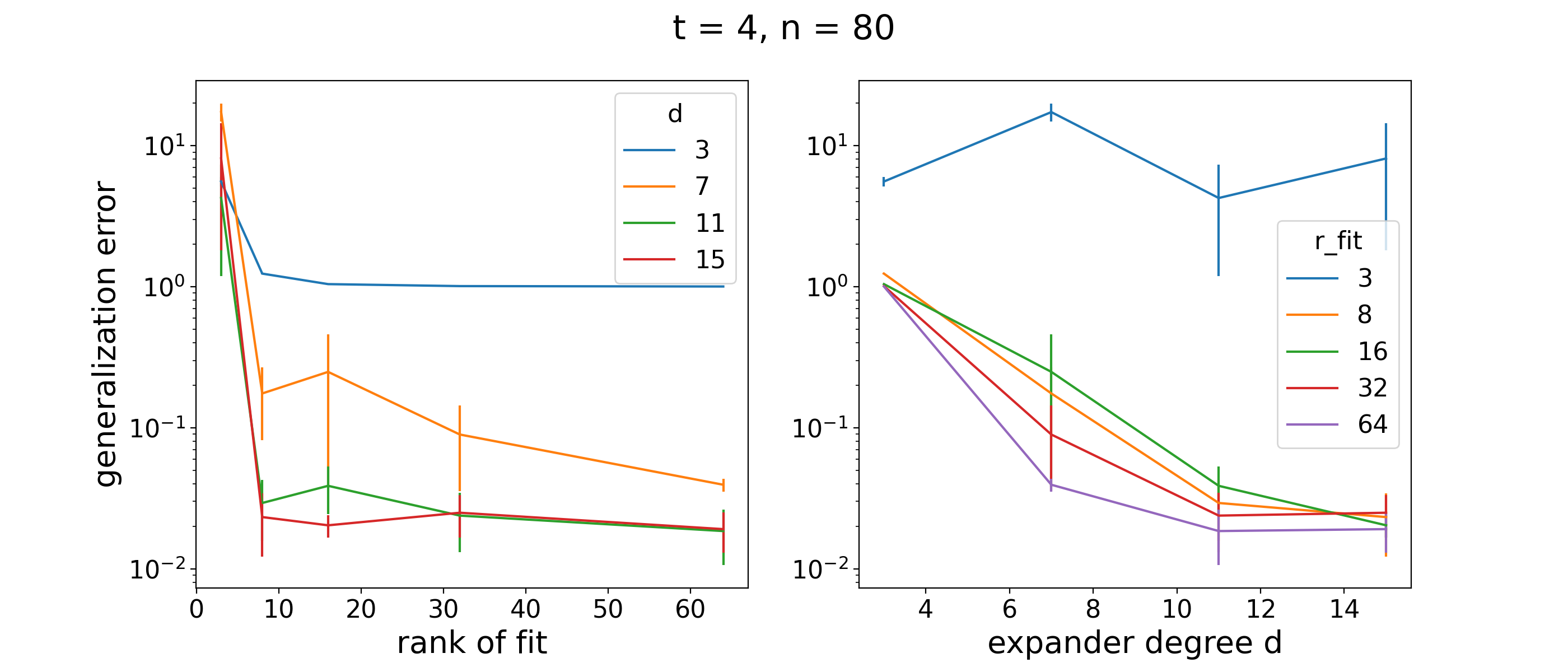}
    \caption{Relative error of reconstruction for $t=4$ and $n=40,80$.
        The results are average of 6 tensors with standard errors shown.}
    \label{fig:dim_4_results}
\end{figure}

\begin{figure}[t]
    \centering
    \includegraphics[width=.6\linewidth]{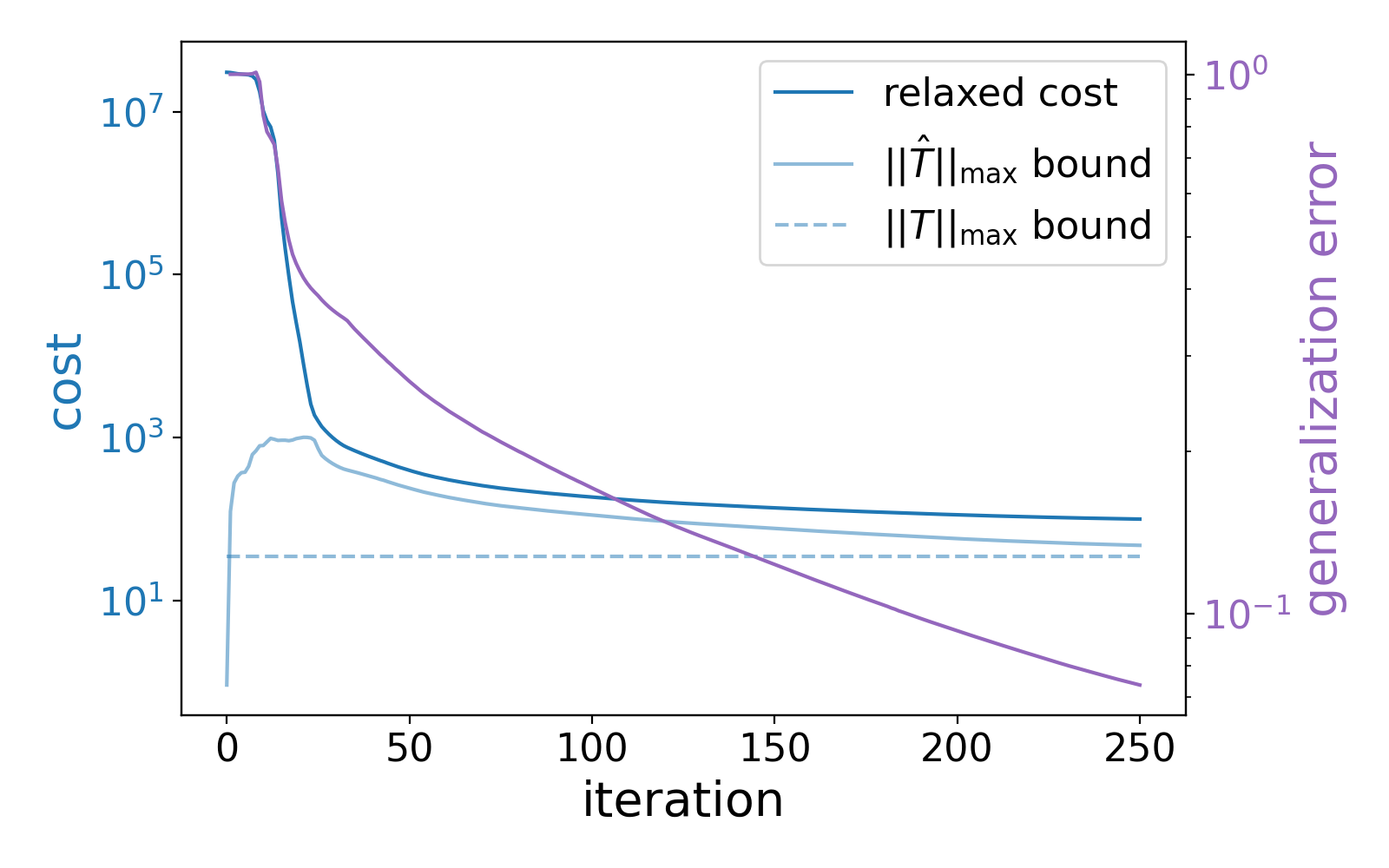}
    \caption{The behavior of the optimization algorithm by iteration for a test case with 
    $t=4, n=400, d=12, r_{\rm fit}=16$.
    Initially, the algorithm focuses on fitting the data, 
    causing the max-qnorm to increase until the 
    residuals are comparable to $\delta$, at which point the cost is mostly due to 
    the max-qnorm.
    In the final stage, the max-qnorm is decreased,
    leading to a decrease in the generalization error.
    Note that the max-qnorm lines are upper bounds,
    and the algorithm hasn't fully converged in 250 iterations.
    }
    \label{fig:iterations}
\end{figure}

We generate random rank $r=3$, order $t$ tensors
by drawing the entries of $U^{(i)}$ 
independently from the uniform distribution $\mathrm{Unif}[-1,1]$
for $i=1,\ldots, t$. 
The resulting tensor $T$ is rescaled to have $\| T\|_F = \sqrt{n d^t}$,
so that the root-mean-square of its entries is 1.
Entries are sampled using a random $d$-regular expander graph
as described in Section~\ref{sec:regular}, and no noise was added.
In principle, we could have used one of any number of deterministic
$d$-regular expander constructions, but these are more difficult
to implement for arbitrary $d$.
Simulations were run for parameter ranges 
$n \in \{ 20, 40, 80\}, 
t \in \{ 3, 4 \}, 
d \in \{ 3, 7, 11, 15 \}$.
We fit with tensors of rank 
$r_{\rm fit} \in \{ 3, 8, 16, 32, 64 \}$,
residual parameter $\delta = 0.05 \sqrt{|E|}$,
with $\kappa = 100$, $\beta = 1$.
We report generalization error, defined as 
\begin{equation*}
    \mbox{generalization error} := 
    \frac{\| \hat{T} - T \|_F}{\| T \|_F}.
\end{equation*}

The results for and $t=4$ are shown in Figure~\ref{fig:dim_4_results},
with a similar Figure~\ref{fig:dim_3_results}
for $t=3$ in the Appendix.
In either dimension, errors $<$10\% are achieved 
whenever $d$ and $r_{\rm fit}$ are sufficiently high.
The fraction of observed entries in this setting is 
$(d/n)^{t-1}$.
For sufficiently large $r_{\rm fit}$, we see error $<$10\% once
$d \geq 11$, or observing 
2\% ($n=40, t=3$), 0.6\% ($n=40, t=4$), 
0.3\% ($n=80, t=3$), and 0.04\% ($n=80, t=4$) 
of the total entries.

It is interesting to note that having $r_{\rm fit} \gg r$
helps the optimization algorithm find a better solution.
Furthermore, having a highly overparameterized model
does not appreciably hurt its generalization ability.
Once the rank of the fit is above a threshold around 10, the error 
stays essentially flat.
Overparametrizing the rank appears to help the optimization
find a good solution, while the max-qnorm penalty 
controls the complexity of this solution to prevent overfitting.

When we compare 
$\prod_{i=1}^t \|U^{(i)}\|_{2,\infty} $
and 
$\prod_{i=1}^t \|\hat U^{(i)}\|_{2,\infty} $,
which are upper bounds on the max-qnorm of the truth $T$ and 
estimate $\hat{T}$,
we find that the estimate generalizes well
when its upper bound is less than or equal to that of the truth.
This is evident in Figure~\ref{fig:iterations},
where we see how the cost and generalization error
change with iteration for $t=4, n=400, d=12, r_{\rm fit}=16$
and otherwise the same parameters as before.
This corresponds to a fraction $2.7\times 10^{-5}$
of the total entries observed.
The overall relaxed cost \eqref{eq:cost_mod}
is shown along with the max-qnorm upper bounds.
Initially, the cost is dominated by the
residual term. 
The residuals decrease while  
the bound on $\|\hat{T}\|_{\rm max}$ actually
increases.
Once the mean-square residuals are as small as $\delta^2$,
the cost is dominated by the max-qnorm complexity term
and the $\|\hat{T}\|_{\rm max}$ bound decreases.
These stages can be understood by the variable
projection parameter varying between
$\mu \approx 0$ and $\mu \approx 1$;
see Section~\ref{sec:var_proj}.
Throughout this process the generalization error
decreases, although it has not fully converged in 250 iterations.

Finally, we also have performed an experiment with
$n=1000,
t = 4,
d = 12,
r_{\rm fit} = 16$ and otherwise the same parameters as before.
After running for 271 iterations, we stopped the optimization 
and found generalization error of 3.4\%.
This is a very small fraction
$1.7\times10^{-6}$ 
of the total entries in $T$.
This and the $n=400$ experiment show that 
good recovery is possible even when $d \ll n$.

\section{Discussion}
\label{sec:discussion}

We have deterministically analyzed tensor completion
using the max-qnorm as a measure of complexity
and hypergraph sampling.
Our main results show that,
by finding the tensor with smallest max-qnorm
that is consistent with the observations,
one may obtain a good estimate of the true tensor.
The error of this estimate depends on the expansion
properties of the hypergraph model.
Auxiliary to this main result are a number of newly proven
facts about the max-qnorm which may be of interest
to specialists in communication complexity.

We show that proximal algorithms based on a relaxation of
constrained max-qnorm minimization
are practical to implement and provide code to reproduce our results.
Although our numerical study is mainly proof-of-concept,
the method does work in the problem sizes that were tested.
These were mainly small-scale ($n<100$) 
but promising results were found in a medium scale tests ($n=400,1000$)
with very small fractions of the 
total entries observed.
Sampling based on $d$-regular expander graphs,
as opposed to sampling tensor entries uniformly at random,
was successful in these experiments.

A number of theoretical and practical considerations still remain.
Theoretically, it would be nice to have other
constructions of sparse hypergraph expanders.
Completely deterministic constructions could be
useful for applications, e.g.\ for compressive sensing in hardware
where a single sensor reuses its ``observation mask'' over and over.

Finding the right complexity measures for tensors is still an open problem.
Here we study max-qnorm because it amenable to deterministic analysis,
however there are likely other measures that work better in some settings.
For example, if the factors are known to be smooth,
e.g.\ due to spatial autocorrelation in images,
then other regularization will be beneficial \cite{park2020a}.

The optimization problem is non-convex,
but our simulations show that good
solutions can be found when the fit uses an overparameterized rank.
It may be possible to explain this using the techniques of
\cite{haeffele2015global,liu2020toward}.

A thorough numerical comparison of the performance of max-qnorm
minimization versus simple Frobenius norm minimization
would be informative but outside the scope of this paper.
Simulations of larger size will be needed to see whether the
numerical results support the $O(n)$ sample complexity that
the theory predicts.
Furthermore, our experiments were with incoherent, random 
tensors which could make the problem easier.
It would also be interesting to try expander sampling with
real datasets.
It will be important to reconcile various conjectured
hardness results \cite{barak2016noisy}
with the practical success of these algorithms
in many settings.

\subsection*{Acknowledgements}

Thank you to Ioana Dumitriu for support and discussions. 
We are grateful to Paul Beame 
for discussions of communication complexity and to Adi Shraibman
for sharing details of the proof of the Hadamard product 
bound on the max-norm for matrices.
K.D.H.\ was supported by a Washington Research Foundation Postdoctoral Fellowship.
Y.Z.\ was partially supported by NSF DMS-1949617.

\bibliographystyle{plainnat}
\bibliography{ref.bib}

\appendix

\section{Supplemental proofs}

\label{sec:appendix}





\subsection{Proof of Lemma \ref{lem:tensormixing}} \label{sec:lemmix}

We write the adjacency tensor $T_H$ as $T$ for simplicity.
Let $\mathbf{1}_{W_i}$ be the indicator vector of the set $W_i\subset V_i$ such that the $j$-th entry of $\mathbf{1}_{W_i}$ is $1$ if $j\in W_i$ and $0$ if $j\not\in W_i$. For any vectors $u_1,\dots, u_t\in \mathbb R^n$, define
\begin{align*}
    T(u_1,\dots,u_t)=\sum_{i_1,\dots,i_t\in [n]} T_{i_1,\dots,i_t}u_1(i_1)\cdots u_t(i_t). 
\end{align*}
We then have 
\begin{align*}
   \frac{\left|e(W_1,\dots,W_t)-\frac{|E|}{n^t}|W_1|\cdots |W_t| \right|}{\sqrt{|W_1|\cdots |W_t|}} &=\frac{\left|T(\mathbf{1}_{W_1},\dots, \mathbf{1}_{W_t})-\frac{|E|}{n^t}J(\mathbf{1}_{W_1},\dots, \mathbf{1}_{W_t}) \right|}{\sqrt{|W_1|\cdots |W_t|}}\\
   &\leq \left\|T-\frac{|E|}{n^t}J\right\|= \lambda_2(H),
\end{align*}
where the last inequality is from the definition of the spectral norm in \eqref{eq:defspecT} and $\lambda_2(H)$ in \eqref{eq:lambda2H}.
This completes the proof.

\subsection{Proof of Theorem \ref{thm:main2}}\label{sec:proof_quasiregular_completion}
Consider a tensor $T$ such that  
$T \in \R^{n_1\times \cdots \times n_t} $.
We sample the entry $T_e$ whenever 
$e = (i_1,\dots, i_t)$ is a hyperedge in $H$ defined in Section \ref{sec:quasiregular}.  
Then the sample size is $|E|=n_1\prod_{i=1}^{t-1}d_{2i-1}$.  
Consider a rank-1 sign tensor
  $S = s_1 \circ \cdots \circ s_t$
  with $s_j \in \{\pm 1\}^{n_j}$. Following the same steps in the proof of Theorem \ref{thm:main3}, we obtain
  \begin{align*}
    \left| 
    \frac{1}{\prod_{i=1}^t n_i} \sum_{e \in \prod_{i=1}^t[n_i]} 
    S_{e} 
    - 
    \frac{1}{|E|} \sum_{e \in E} 
    S_{e} 
    \right|
    &\leq 2 
      \sum_{j=1}^{2^{t-1}}
      \left| \frac{|W_{1,j}| \cdots |W_{t,j}|}{\prod_{i=1}^t n_i} - 
      \frac{ e(W_{1,j}, \ldots, W_{t,j}) }{|E|} \right| .
  \end{align*}
  Applying Theorem~\ref{thm:mixingbipartite} to the sets $W_{1,j}\subseteq V_1,\dots, W_{t,j}\subseteq V_t$ for each $1\leq j\leq 2^{t-1}$, we get that
  \begin{align*}
    \left| \frac{1}{\prod_{i=1}^t n_i} \sum_{e \in \prod_{i=1}^t[n_i]} S_e 
    - \frac{1}{|E|} \sum_{e \in E} S_e \right|
    &\leq 2^t \left(\frac{\lambda^{(1)}}{4\sqrt{d_1d_2}}+\sum_{k=2}^{t-1}\frac{\lambda^{(k)}}{2\sqrt{d_{2k-1}d_{2k}}}\right)
    .
    \end{align*}
  We now write the tensor 
  $T = \sum_i \alpha_i S_i$ 
  as a sum of rank-1 sign tensors $S_i$, with coefficients
  $\alpha_i \in \mathbb{R}$. 
Define $R=(\hat{T}-T)*(\hat{T}-T)$.
 Following the same steps in the proof of Theorem \ref{thm:main3}, we have
  \begin{align*}
 \frac{1}{\prod_{i=1}^t n_i}\|\hat{T}-T\|_F^2&= \left| \frac{1}{\prod_{i=1}^t n_i} \sum_{e \in \prod_{i=1}^t [n_i]} R_e \right| \\
  &\leq 2^{t+1}K_G^{t-1} \left(\frac{\lambda^{(1)}}{2\sqrt{d_1d_2}}+\sum_{k=2}^{t-1}\frac{\lambda^{(k)}}{\sqrt{d_{2k-1}d_{2k}}}\right)
  \|T \|_\mathrm{max}^2 .
  \end{align*}
This completes the proof.

\subsection{Proof of Lemma \ref{lem:quasiregular}}
\label{sec:lemdegree}
Each vertex in $V_1$ has degree $d_1d_3\cdots d_{2t-3}$ by our construction. In general, for each vertex $w_i\in V_i$, the degree of $w_i$ is the number of walks $(v_1,\dots, v_{i-1},w_i, v_{i+1},\cdots,v_t)$ for all possible $v_1,\dots,v_{i-1},v_{i+1},\dots,v_t$. To form a walk of length $t$, there are $\prod_{k=1}^{i-1}d_{2k}$ many ways to choose $(v_1,\dots, v_{i-1})$ and $\prod_{k=i}^{t-1}d_{2k-1}$ many ways to choose $(v_{i+1},\cdots,v_t)$. Then Lemma \ref{lem:quasiregular} holds.

\subsection{Proof of Theorem \ref{thm:mixingbipartite}}\label{sec:appendix_mixbipartite}

We first prove the following lemma for random walks on bipartite biregular graphs.
\begin{lemma}\label{lem:SRW}
Suppose that $G=(V_1,V_2,E)$ is a $(d_1,d_2)$-biregular  bipartite graph on vertex sets $V_1=[n_1]$ and $V_2=[n_2]$.
Let $X_1$ be a vertex uniformly distributed on $V_1$, and let $X_2$
be given by walking from $X_1$ along an adjacent edge uniformly. Then $X_2$
is uniformly distributed on $V_2$.
\end{lemma}
\begin{proof}
Since every vertex in $V_1$ has degree $d_1$, for any $v_1\in [n_1], v_2\in [n_2]$,
\begin{align*}
    \mathbb P(X_2=v_2\mid X_1=v_1)=\frac{\mathbf{1}\{(v_1,v_2)\in E\}}{d_1}.
\end{align*}
Then 
\begin{align*}
    \mathbb P(X_2=v_2) &= 
   \sum_{v_1\in [n_1]}\mathbb P(X_2=v_2\mid X_1=v_1)\mathbb P(X_1=v_1) \\
    &=\frac{1}{n_1}\sum_{v_1\in [n_1]}\frac{\mathbf{1}\{(v_1,v_2)\in E\}}{d_1}=\frac{d_2}{n_1d_1}=\frac{1}{n_2}.
\end{align*}
Therefore $X_2$ is uniformly distributed in $V_2$.
\end{proof}

Let $(X_1,\dots, X_t)$ be a simple random walk of length
 $t-1$  defined as follows. 
 Let $X_1$ be a uniformly chosen starting vertex from $V_1$, 
 and for $2\leq i\leq t$, let
 $X_i$ be a uniformly chosen neighbor of 
 $X_{i-1}$ in the bipartite biregular graph $G_{i-1}$. From Lemma \ref{lem:SRW}, $X_1,\dots, X_t$ are uniformly distributed in $V_1,\dots,V_t$.
  Then from the construction of $H$, and the definition of simple random walks on graphs, we have 
 \begin{align}\label{eq:mix_prop2}
 \frac{e(W_1,\dots, W_t)}{n_1d_1d_3\cdots d_{2t-3}} = 
    \mathbb P(X_1\in W_1,\dots, X_t\in W_t).
\end{align}
Let
\begin{align*}
    A_i=\begin{bmatrix}
     0 & B_i^{\top}\\
    B_i & 0 \end{bmatrix}
\end{align*}
be the adjacency matrix of $G_i$, where $B_i\in \mathbb R^{n_2\times n_1}$ is the  matrix such that for any $v_2\in [n_2],v_1\in [n_1]$, $(B_i)_{v_2,v_1}=1$ if and only if $v_1,v_2$ are connected in $G_i$ and $0$ otherwise.  The largest eigenvalue of $A_i$ is $\sqrt{d_{2i-1}d_{2i}}$. Let $\lambda^{(i)}$ be the second largest eigenvalue of $A_i$. Define 
\begin{align*}
    M_i=\frac{1}{d_{2i-1}} B_i.
\end{align*}
Then the largest  singular value  of $M_i$ is $\sqrt{\frac{d_{2i}}{d_{2i-1}}}$, and the second largest singular value is $\frac{\lambda^{(i)}}{d_{2i-1}}$. 

Denote $\1_{n}=(1,\dots, 1)^\top\in \mathbb R^{n}$.
We have 
\begin{align}\label{eq:M_avg}
    M_i \left(\frac{1}{n_i}\1_{n_i}\right)&=\frac{1}{d_{2i-1}}B_i\left(\frac{1}{n_i}\1_{n_i}\right)=\frac{1}{n_{i+1}} \1_{n_{i+1}},  \quad 
     \frac{1}{n_{i+1}} \1_{n_{i+1}}^\top M_i=\frac{1}{n_i}\1_{n_i}^{\top}.
\end{align}
Let  $P_i\in \mathbb R^{n_i\times n_i}$ be the projection map on to the linear subspace of $\mathbb R^{n_i}$ spanned by $\{e_i, i\in W_i\}$, where $\{e_i, 1\leq i\leq n_i\}$ are the  standard basis in $\mathbb R^{n_i}$. We have 
\begin{align*}
    \mathbb P(X_1\in W_1,\dots, X_t\in W_t)=\left\| P_t M_t P_{t-1} M_{t-1} \cdots P_2 M_1 P_1 \left(\frac{1}{n_1}\1_{n_1}\right)\right\|_1. 
\end{align*}
Let $ v_1=P_1 \left(\frac{1}{n_1}\1_{n_1}\right)$ and for all $2\leq i\leq t$,
\begin{align*}
    \quad v_i:=P_iM_{i-1}v_{i-1}.
\end{align*}
Decompose $v_i=x_i+y_i$, where $x_i$ is the part of $v_i$ that is a scalar multiple of $\1_{n_i}$, and $y_i$ is the part of $v_i$ that is orthogonal to $\1_{n_i}$. We first prove the following lemma.
\begin{lemma}\label{lem:bipartite1} For $1\leq i\leq t-1$,
\[ \left|\|v_{i+1}\|_1- \alpha_{i+1} \|v_{i}\|_1 \right|\leq \frac{\lambda^{(i)}}{d_{2i-1}} \sqrt{n_{i+1}\alpha_{i+1}(1-\alpha_{i+1})}\|y_i\|_2.
\]
\end{lemma}
\begin{proof}
\begin{align*}
\|v_{i+1}\|_1&=P_{i+1}M_{i}v_{i}=\1_{W_{i+1}}^{\top} M_{i}v_i\\
&=\alpha_{i+1}\1_{n_{i+1}}^{\top} M_{i} \left(\frac{\|v_{i}\|_1}{n_i} \1_{n_{i}}\right)+ \left(\1_{W_{i+1}}-\alpha_{i+1}\1_{n_{i+1}}\right)^{\top} M_{i} \left(v_i-\frac{\|v_{i}\|_1}{n_i} \1_{n_i}\right)\\
    &\quad +(\1_{W_{i+1}}-\alpha_{i+1}\1_{n_{i+1}})^{\top} M_i \left( \frac{\|v_{i}\|_1}{n_i} \1_{n_i}\right)+\alpha_{i+1}\1_{n_{i+1}}^{\top} M_i \left(v_i-\frac{\|v_{i}\|_1}{n_i} \1_{n_i}\right).
\end{align*}
From \eqref{eq:M_avg},
\begin{align*}
  & (\1_{W_{i+1}}-\alpha_{i+1}\1_{n_{i+1}})^{\top} M_i \left(\frac{\|v_{i}\|_1}{n_i} \1_{n_i}\right)=0,\quad \alpha_{i+1}\1^{\top}_{n_{i+1}} M_i \left(v_i-\frac{\|v_{i}\|_1}{n_i} \1_{n_i}\right)=0.
\end{align*}
We obtain 
\[  \|v_{i+1}\|_1
=\alpha_{i+1} \|v_i\|_1+ \left(\1_{W_{i+1}}-\alpha_{i+1}\1_{n_{i+1}}\right)^{\top} M_i \left(v_i-\frac{\|v_{i}\|_1}{n_i} \1_{n_i}\right).
\]
Therefore by Cauchy inequality,  $\left|  \|v_{i+1}\|_1-\alpha_{i+1} \|v_i\|_1\right|$ can be bounded by 
\begin{align*}
  \left| \left(\1_{W_{i+1}}-\alpha_{i+1}\1_{n_{i+1}}^{\top}\right) M \left(v_i-\frac{\|v_{i}\|_1}{n_{i}}\1_{n_{i}}\right)\right|
  &\leq \frac{\lambda^{(i)}}{d_{2i-1}} \|\1_{W_{i+1}}-\alpha_{i+1}\1_{n_{i+1}}^{\top}\|_2 \left\|v_i-\frac{\|v_{i}\|_1}{n_{i}}\1_{n_{i}} \right\|_2 \\
  &=\frac{\lambda^{(i+1)}}{d_{i+1}}\sqrt{n_{i+1}\alpha_{i+1}(1-\alpha_{i+1}) }\|y_i\|_2.
\end{align*}
\end{proof}

Now we are ready to prove Theorem \ref{thm:mixingbipartite}.
\begin{proof}[Proof of Theorem \ref{thm:mixingbipartite}]
We have $\|y_1\|_2=\frac{\sqrt{\alpha_1(1-\alpha_1)}}{\sqrt{n_1}}$ and $\|v_1\|_2=\frac{\sqrt{\alpha_1}}{\sqrt n_1}$. For $i\geq 2$,
\begin{align*}
    \|y_i\|_2\leq \|v_i\|_2=\|P_iM_{i-1}v_{i-1}\|_2\leq \sqrt{\frac{d_{2i-2}}{d_{2i-3}}}\|v_{i-1}\|_2,
\end{align*}
which implies
\begin{align}
     \|y_i\|_2&\leq \|v_i\|_2\leq \frac{\sqrt{\alpha_1}}{\sqrt n_1} \prod_{k=1}^{i-1}\sqrt{\frac{d_{2k}}{d_{2k-1}}} \notag\\
     &=\frac{\sqrt{\alpha_1}}{\sqrt{n_1}}\prod_{k=1}^{i-1}\sqrt{\frac{d_{2k}}{d_{2k-1}}}=\frac{\sqrt{\alpha_1}}{\sqrt n_i}=\frac{\sqrt{\alpha_1}}{\sqrt {n_{i+1}}}\frac{\sqrt{d_{2i-1}}}{\sqrt{d_{2i}}},\label{eq:y_i}
\end{align}
where in the last equation we use the condition $n_id_{2i-1}=n_{i+1}d_{2i}$. From Lemma \ref{lem:bipartite1} and \eqref{eq:y_i}, we obtain the following inequalities: 
\begin{align*}
    \left|\|v_{t}\|_1 -\alpha_t \|v_{t-1}\|_1  \right|
    &\leq \frac{\lambda^{(t-1)}}{\sqrt{d_{2t-2}d_{2t-3}}} \sqrt{\alpha_1\alpha_{t}(1-\alpha_{t})},\\ 
    \left|\alpha_t\|v_{t-1}\|_1 -\alpha_{t-1}\alpha_t \|v_{t-2}\|_1 \right|
    &\leq \frac{\lambda^{(t-2)}}{\sqrt{d_{2t-4}d_{2t-5}}} \sqrt{\alpha_1\alpha_{t-1}(1-\alpha_{t-1})},\\
     &\vdots\\
    \left| \alpha_4 \cdots \alpha_t \|v_3\|_1 - \alpha_3 \cdots \alpha_t \|v_2\|_1 \right|
     &\leq \frac{\lambda^{(2)}}{\sqrt{d_4d_{3}}}(\alpha_4\cdots\alpha_t) \sqrt{\alpha_1\alpha_3 (1-\alpha_3)} ,\\
\left|\alpha_3\cdots\alpha_t\|v_2\|_1-\alpha_2\cdots\alpha_t\|v_1\|_1\right|
&\leq \frac{\lambda^{(1)}}{\sqrt{d_2d_1}}(\alpha_3\cdots\alpha_t) \sqrt{\alpha_1(1-\alpha_1)\alpha_2(1-\alpha_2)}.
\end{align*}
Since $\|v_1\|_1=\alpha_1$, combining  the inequalities above, using the triangle inequality, we obtain for $t\geq 2$,
\begin{align*}
  \left| \|v_{t}\|_1-\alpha_1\alpha_2\cdots \alpha_t \right|
   \leq & \frac{\lambda^{(1)}}{\sqrt{d_1d_2}}\sqrt{\alpha_1(1-\alpha_1)\alpha_2(1-\alpha_2)}\alpha_3\cdots\alpha_t\\
   &+\sum_{k=2}^{t-1}\frac{\lambda^{(k)}}{\sqrt{d_{2k-1}d_{2k}}}\sqrt{\alpha_1\alpha_{k+1}(1-\alpha_{k+1})}\alpha_{k+2}\cdots \alpha_t\\
  \leq & \frac{\lambda^{(1)}}{4\sqrt{d_1d_2}}+\sum_{k=2}^{t-1}\frac{\lambda^{(k)}}{2\sqrt{d_{2k-1}d_{2k}}}.
\end{align*}
This completes the proof.
\end{proof}

\subsection{Proof of Theorem \ref{thm:max-qnorm-properties}}
\label{pf:max-qnorm-properties}
We prove claims (1--4) in order.
For claim (1), 
write $T$ using the decomposition
which attains the max-qnorm,
$T = \bigcirc_{i=1}^t U^{(i)}$
for some $U^{(i)} \in \R^{n_i \times r}, 1\leq i\leq t$.
Then a single entry of $T$ can be written as 
\[T_{i_1, \ldots, i_t} 
= 
\sum_{i=1}^r u^{(1)}_{i_1, i} \cdots u^{(t)}_{i_t, i},\]
and we have that
$T_{I_1, \ldots, I_t} = \bigcirc_{i=1}^t U^{(i)}_{I_i, :}$, where $U^{(i)}_{I_i,:}$ denotes the submatrix of $U^{(i)}$ with the column restricted on $I_i$.
Since the norm $\| A \|_{2,\infty}$
is  non-increasing 
under removing any rows of a matrix $A$,  $\|U_{I_i,;}^{(i)}\|_{2,\infty}\leq \|U^{(i)}\|_{2,\infty}$. 
Since $T_{I_1, \ldots, I_t}$
can be factored by selecting
subsets of the rows of $U^{(i)}$, by the definition of max-qnorm, we have \begin{align*}
    \|T_{I_1,\dots,I_t}\|_{\max}\leq \prod_{i=1}^t \|U_{I_i,;}^{(i)}\|_{2,\infty}\leq \prod_{i=1}^t\|U^{(i)}\|_{2,\infty}=\|T\|_{\max}.
\end{align*}
This proves claim (1). For claim (2),
let \[T = \bigcirc_{i=1}^t T^{(i)}
\quad \text{ and }\quad  S = \bigcirc_{i=1}^t S^{(i)}\]
be the rank $r_1$ and $r_2$ decompositions of $T$ and $S$
that attain their max-qnorms.
Then since
\begin{align*}
(T \otimes S)_{k_1, \ldots, k_t} 
&=
T_{i_1, \ldots, i_t} S_{j_1, \ldots, j_t} =
\left(
\sum_{l=1}^{r_1} T^{(1)}_{i_1, l} \cdots T^{(t)}_{i_t, l} 
\right)
\left(
\sum_{l'=1}^{r_2} S^{(1)}_{j_1, l'} \cdots S^{(t)}_{j_t, l'}
\right) \\
&=
\sum_{l=1}^{r_1}
\sum_{l'=1}^{r_2}
\left( T^{(1)}_{i_1, l} S^{(1)}_{j_1, l'} \right)
\cdots
\left( T^{(t)}_{i_t, l} S^{(t)}_{j_t, l'} \right) \\
&=
\sum_{p=1}^{r_1 r_2}
\left( T^{(1)} \otimes S^{(1)} \right)_{k_1, p}
\cdots
\left( T^{(t)} \otimes S^{(t)} \right)_{k_t, p}
\end{align*}
for $k_s = j_s + m_s (i_s - 1)$
for all $s = 1, \ldots, t$
and $p = l' + r_2 ( l - 1)$,
we have that
\[
T \otimes S = \bigcirc_{i=1}^t (T^{(i)} \otimes S^{(i)}).\]
Note that any matrices $A \in \R^{m \times n}$ and
$B \in \R^{p \times q}$,
\begin{align}\label{eq:matrix2infinity}
\|A \otimes B\|_{2,\infty} = \|A \|_{2,\infty} \|B\|_{2,\infty} .
\end{align}
To see this, assume
without loss of generality that the rows of $A$ and $B$
with greatest $\ell_2$-norm are the first
(all combinations of rows occur in the Kronecker product).
Then the first row of $A \otimes B$ will have the largest
$\ell_2$-norm of all rows in that matrix;
call it $x$.
Therefore, $\| A \otimes B \|_{2,\infty}^2
= \| x \|_2^2$, and
\[
\| x \|_2^2 
= \sum_{i=1}^n \sum_{j=1}^q (A_{1,i} B_{1,j})^2
= \sum_{i=1}^n A_{1,i}^2 \| B_{1,:} \|_2^2
= \| A_{1,:} \|_2^2 \, \| B_{1,:} \|_2^2 
= \|A \|_{2,\infty}^2 \|B\|_{2,\infty}^2 .
\]
This implies that
$\| T^{(i)} \otimes S^{(i)} \|_{2,\infty} = 
\| T^{(i)} \|_{2,\infty} \| S^{(i)} \|_{2,\infty}$ for $1\leq i\leq t$, therefore
\begin{align*}
    \|T\otimes S\|_{\max}\leq \prod_{i=1}^t \|T^{(i)}\otimes S^{(i)}\|_{2,\infty}=\prod_{i=1}^t\left(\| T^{(i)} \|_{2,\infty} \| S^{(i)} \|_{2,\infty}\right)=\|T\|_{\max}\|S\|_{\max}. 
\end{align*}
This completes the proof of  claim (2). For claim (3), note that every entry in 
$T *S$ 
appears in 
$T \otimes S$, since
\[
(T *S)_{i_1, \ldots, i_t} 
= 
(T \otimes S)_{i_1 + n_1 (i_1 - 1), \ldots, i_t + n_t(i_t - 1)} .
\]
So we have that $T *S = (T \otimes S)_{I_1, \ldots, I_t}$ for some subsets of indices $I_1,\dots, I_t$,
and by claim (1)
the result follows. 
Finally, since from claim (2) and (3),
$
    \|T*T\|_{\max}\leq \|T\otimes T\|_{\max}\leq \|T\|_{\max}^2,
$ claim (4) follows.

\subsection{Proof of Proposition \ref{prop:incoherent}}\label{sec:proof_of_incoherent}

Take the rank $r$ factorization given by the definition of incoherence.
Then $\|U^{(i)} \|_{2,\infty} \leq \sqrt{r} \cdot |U^{(i)}|_\infty$
by an elementary inequality for the 2-norm of the length $r$ rows.
Thus $\|T\|_{\max} \leq \prod_{i=1}^t \|U^{(i)} \|_{2,\infty} \leq (C \sqrt{r})^t$.

\subsection{Proof of Proposition \ref{thm:improved}}
\label{pf:improved}
  Consider the rank-$r$ factorization $T = \bigcirc_{i=1}^{t} U^{(i)}$
  that attains the max-qnorm.
  Then,
  \begin{align}
      \|T \|_\mathrm{max} &= 
      \| U^{(1)} \|_{2,\infty} \| U^{(2)} \|_{2,\infty} \cdots \| U^{(t)} \|_{2,\infty} \nonumber \\
      &= \| U^{(1)} \|_{2,\infty} \| U^{(2)} \otimes \cdots \otimes U^{(t)} \|_{2,\infty} \geq \| U^{(1)} \|_{2,\infty} \| B \|_{2,\infty}, \label{eq:max_lb_matrix1}
  \end{align}
  where the second equality is from \eqref{eq:matrix2infinity}. And in the last inequality, $B = (U^{(2)} \otimes \cdots \otimes U^{(t)})_{:, I}$ is any submatrix 
  obtained by taking some subset of columns $I \subseteq [r^{t-1}]$.
  On the other hand, the mode-1 flattening may be written as 
  \citep{kolda2009tensor}
  \begin{equation}
  T_{[1]} = 
  \sum_{i=1}^r U^{(1)}_{:,i} 
  \left( U^{(2)}_{:,i} \otimes \cdots \otimes U^{(t)}_{:,i} \right)^{\top}
  = U^{(1)} \left( U^{(2)} \odot \cdots \odot U^{(t)} \right)^\top  .
  \label{eq:max_lb_matrix2}
  \end{equation}
  The symbol ``$\odot$'' is the Khatri-Rao product of matrices, a.k.a.\ the
  ``matching columnwise'' Kronecker product:
  For 
  $A \in \R^{m \times n}$
  and
  $B \in \R^{l \times n}$,
  define
  $A \odot B$
  as the $ml \times n$
  matrix with 
  entries
  $(A \odot B)_{ij} = A_{aj} B_{bj}$
  for $i = b + l (a - 1)$.
  Take $B=U^{(2)} \odot \cdots \odot U^{(t)}$ in \eqref{eq:max_lb_matrix1},
  then by \eqref{eq:max_lb_matrix2} we have that
  $T_{[1]} = U^{(1)} B^T$
  is a valid factorization, which shows that 
  $\|T_{[1]} \|_\mathrm{max} \leq \| T \|_\mathrm{max}.$
  Flattening over any other mode is equivalent, 
  so the first part of the inequality holds. 
  Since the matrix $T_{[i]}$ and the tensor 
  $T$ contain the same values, then
  by \eqref{eq:max-rank-matrix},
  $
  \|T_{[i]}\|_{\max} 
  \geq 
  \|T_{[i]}\|_{1,\infty}
  =
  \max_{i_1, \ldots, i_t} |T_{i_1,\ldots, i_t}|$.  
  This completes the proof.


\subsection{Proof of Lemma \ref{lem:grothendieck}}
\label{pf:grothendieck}

We will use the following multilinear extension of Grothendieck's inequality in \cite{perez2006trace},  which improves the constant in \cite{blei1979multidimensional}.
There is also a more general version in Corollary 4.4 of \cite{bombal2004multilinear}, where the authors also showed that the constant $K_G^{t-1}$ for an order-$t$ tensor is optimal.
This leads to improved constants over \citet{ghadermarzy2018}.

\begin{theorem}[\cite{perez2006trace}, Theorem 2.2]\label{thm:newG}
Let $T\in \mathbb R^{n^t}$ be an order-$t$. Let $k$ be a positive integer and let $u_{i_j}^j\in \mathbb R^k$, $1\leq j\leq t, 1\leq i_j\leq n$ be  vectors such that $\|u_{i_j}^j\|_2\leq 1.$ Then
\begin{align}
    &\left|\sum_{i_1,\dots,i_t=1}^n T_{i_1,\dots,i_t}\sum_{s=1}^k u_{i_1}^1(s)\cdots u_{i_t}^t(s)\right|  \notag\\
    \leq & K_G^{t-1} \sup_{\|x_1\|_{\infty},\dots,\|x_t\|_{\infty}\leq 1} \left| \sum_{i_1,\dots,i_t=1}^n T_{i_1,\dots,i_t} x_1(i_1)\cdots x_t(i_t)\right|.
\end{align}
\end{theorem}

Let 
$
\mathbb{B}_\pm(1) = 
\mathrm{conv} ( 
\{T: 
T\in \{\pm 1\}^{n \times \cdots \times n}, 
\mathrm{rank}(T) = 1 \} 
) 
$
and
$
\mathbb{B}_\textrm{max}(1) = 
\{T: \|T \|_\mathrm{max} \leq 1 \}
$
be the unit balls of the sign nuclear norm
and max-qnorm, respectively. 
We first show the following lemma.
\begin{lemma}\label{lem:ball}
The unit balls of the max-qnorm and sign nuclear norm satisfy
\begin{align}
   \mathbb{B}_\mathrm{max}(1) \subseteq K_G^{t-1} \, \mathbb{B}_\pm(1).
\end{align}
\end{lemma}
\begin{proof} We follow the proof of \cite[Lemma 5]{ghadermarzy2018} with some modification.
Define the inner product between two tensors $T,U\in \mathbb R^{n^t}$ by
\begin{align*}
\langle T, U\rangle :=\sum_{i_1,\dots, i_t=1}^n T_{i_1,\dots,i_t} U_{i_1,\dots,i_t},
\end{align*}
and define
\begin{align}\label{def:Tinfty1}
    \|T\|_{\infty,1}:=\sup_{\|x_1\|_{\infty},\dots,\|x_t\|_{\infty}\leq 1} \left| \sum_{i_1,\dots,i_t=1}^n T_{i_1,\dots,i_t} x_1(i_1)\cdots x_t(i_t)\right|.
\end{align}
The dual norm of the max-qnorm is given by
\begin{align}
    \|T\|_{\max}^*=\sup_{\|U\|_{\max}\leq 1} \langle T, U\rangle =
    \sup_{ 
    \substack{
        \|u_{i_j}^j\|_2\leq 1, u_{i_j}^j\in \mathbb R^k,\\ 
        \forall \; 1\leq j\leq t, \, 1\leq i_j\leq n}
    }
    ~\sum_{i_1, \dots, i_t=1}^n T_{i_1,\dots, i_t}
    \sum_{s=1}^k u_{i_1}^1(s)\cdots u_{i_t}^t(s).
    \label{eq:max_dual}
\end{align}
Using \eqref{def:Tinfty1} and \eqref{eq:max_dual} in Theorem \ref{thm:newG},
we have
\begin{align}\label{eq:T1}
\|T\|_{\max}^*\leq K_G^{t-1}\|T\|_{\infty,1}.    
\end{align} 
Define the unit ball of the max-qnorm dual as
\[\Omega=\left\{Z\in \mathbb R^{n^t},\sup_{\|U\|_{\max}\leq 1} \langle Z,U\rangle \leq 1\right\}.\] 
Taking the dual of the max-qnorm dual leads to
\begin{align}\label{eq:T2}
    (\|T\|_{\max}^*)^*
    &=\sup_{\|Z\|^*_{\max}\leq 1}\langle T, Z\rangle \\
    &=\sup_{Z\in \Omega}\langle T, Z\rangle=\|T\|_{\max} \cdot \sup_{Z\in \Omega}\left\langle \frac{T}{\|T\|_{\max}}, Z\right\rangle \leq \|T\|_{\max}.\notag
\end{align}
Moreover, from \eqref{eq:T1} 
we have that the unit ball of $\|\cdot \|_{\max}$ is larger than
the corresponding ``ball'' of $K_G^{t-1} \| \cdot \|_{\infty,1}$, i.e.
$
    \{ Z: \|Z\|_{\max}^*\leq 1 \}\supseteq \{Z: K_G^{t-1}\|Z\|_{\infty,1}\leq 1\}.
$
This leads to the lower bound
\begin{align}
    (\|T\|_{\max}^*)^*
    &=\sup_{\|Z\|^*_{\max}\leq 1}\langle T, Z\rangle \geq \sup_{K_G^{t-1}\|Z\|_{\infty,1}\leq 1} \langle T,Z\rangle=\frac{1}{K_G^{t-1}} \sup_{K_G^{t-1}\|Z\|_{\infty,1}\leq 1} \langle T, K_G^{t-1}Z\rangle  \notag\\
    &=\frac{1}{K_G^{t-1}} \sup_{\|U\|_{\infty,1}\leq 1} \langle T, U\rangle=\frac{1}{K_G^{t-1}} \|T\|_{\infty,1}^*. \label{eq:dualcomparisionineq}
\end{align}

\eqref{eq:dualcomparisionineq} and \eqref{eq:T2} imply that
\begin{align}\label{eq:Tmaxrelation}
\|T\|_{\max} \geq (\|T\|_{\max}^*)^*\geq \frac{1}{K_G^{t-1}} \|T\|_{\infty,1}^*.
\end{align}

Note that in \eqref{def:Tinfty1}, 
the supremum on the right hand side is achieved when $x_j(i_j)=\pm 1$. 
Therefore for any $T\in \mathbb R^{n^t}$,
 \[\|T\|_{\infty,1}=\sup_{U\in \mathbb B_{\pm}(1)} \langle T, U\rangle .\]
This implies $\mathbb B_{\pm}(1)$ is the unit ball of $\|\cdot \|_{\infty,1}^*$.
From \eqref{eq:Tmaxrelation},
$\mathbb B_{\pm}(1) \subseteq K_G^{t-1} \mathbb B_{\max}(1)$.
\end{proof}

With Lemma \ref{lem:ball} we are ready to prove Lemma \ref{lem:grothendieck}.
\begin{proof}[Proof of Lemma \ref{lem:grothendieck}]
For any tensor $T$, rescaling by $\alpha = \|T\|_\mathrm{max}$ gives $T = \alpha S$, 
where $\| S \|_\mathrm{max} = 1$.
Lemma \ref{lem:ball} implies that $S \in K_G^{t-1} \mathbb{B}_\pm(1)$,
and thus that $\|S\|_\pm \leq K_G^{t-1}$.
Using the scalability property of $\| \cdot \|_\mathrm{max}$ and 
$\| \cdot \|_\pm$ establishes the inequality \[\|T\|_{\pm} \leq K_G^{t-1} \|T\|_{\max}.\] 
\end{proof}

\subsection{Proof of Theorem \ref{thm:max-rank}}
\label{sec:proof_max_rank_property}

We will use the following lemma from \cite{rashtchian2016bounded}.
\begin{lemma}[\cite{rashtchian2016bounded}, Corollary 2.2]\label{lem:roth}
Any rank-$r$ matrix $M\in \mathbb R^{n\times m}$ with $\|M\|_{\infty}\leq 1$ has a factorization $M_{ij}=\langle u_i,v_j\rangle$, where $u_1,\dots,u_{n}\in \mathbb R^r, \|u_i\|_2\leq r^{1/2}, 1\leq i\leq n$ and $v_1,\dots,v_{m}\in \mathbb R^r, \|v_i\|_2\leq 1, 1\leq i\leq m.$
\end{lemma}

Equipped with Lemma \ref{lem:roth}, we  will prove the following lemma inductively, which gives an improvement over Theorem 7 (ii) in \cite{ghadermarzy2018}. The proof is modified from the proof of Lemma 25 in \cite{ghadermarzy2018}.
\begin{lemma}\label{lem:gh23}
Any order-$t$, rank-$r$ tensor $T\in \bigotimes_{i=1}^t \mathbb R^{n_i}$ with $|T|_{\infty}\leq 1$ can be decomposed into $r^{t-1}$ rank-one tensors $T=\sum_{j=1}^{r^{t-1}} u_j^1\circ u_j^2\circ \cdots \circ u_j^t$, where
\begin{align}\label{eq:a31}
    \sum_{j=1}^{r^{t-1}} (u_j^k(s))^2\leq r^{t-1} 
\end{align}
for any $1\leq k\leq t-1$, $1\leq s\leq n_k$, and 
\begin{align}\label{eq:a32}
    \sum_{j=1}^{r^{t-1}} (u_j^t(s))^2\leq r^{t-2}
\end{align}
for any $1\leq s\leq n_t$.
\end{lemma}

\begin{proof}
We do induction on $t$. When $t=2$, Theorem \ref{thm:max-rank} follows from   \ref{lem:roth}.
Now let $T$ be an order-$t$ tensor with $t\geq 3$ such that $|T|_{\infty}\leq 1$ and has the rank-$r$ decomposition as 
\[T=\sum_{j=1}^r v_j^1\circ v_j^2\circ\cdots\circ v_j^t. \]
Matricizing along mode-$1$ we have $T_{[1]}\in \mathbb R^{n_1\times (n_2\cdots n_t)}$ such that 
\begin{align*}
    T_{[1]}=\sum_{i=1}^rv_i^1\circ (v_i^2\otimes \cdots\otimes  v_i^t).
\end{align*}
Let $W$ be a $\mathbb R^{(n_2\cdots n_t)\times r}$ matrix such that $W(:,i)=v_i^2\otimes \cdots\otimes  v_i^t$.
By Lemma \ref{lem:roth}, there exists an $S\in \mathbb R^{r\times r}$ such that 
\begin{align}\label{eq:T1x}
  T_{[1]}=X\circ Y,  
\end{align} where $X=V^{(1)}S\in \mathbb R^{n_1\times r}, Y=WS^{-1} \in \mathbb R^{n_2\cdots n_t\times r}$, and $\|X\|_{2,\infty}\leq r^{1/2},\|Y\|_{2,\infty}\leq 1$. This also implies $\|Y\|_{\infty}\leq \|Y\|_{2,\infty}\leq 1$. Each column of $Y$ is a linear combination of columns of $W$. Hence for some constants $\{\gamma_{ij}\}$,
\[Y(:,i)=\sum_{j=1}^r \gamma_{ij} (v_j^2\otimes \cdots\otimes  v_j^t).\]
Then $E_i:=Y(:,i)$ is a  order-$(t-1)$ tensor of rank at most $r$ with $\|E_i\|_{\infty}\leq 1$. By induction,  we can have a decomposition of $E_i$ as 
\begin{align*}
    E_i=\sum_{j=1}^{r^{t-2}}u_{i,j}^2\circ \cdots \circ u_{i,j}^t,
\end{align*}
where $\sum_{j=1}^{r^{t-2}}(v_{i,j}^k(s))^2\leq  r^{t-2}$ for any $2\leq k\leq t-1$,
$1\leq s\leq N_k$, and
$\sum_{j=1}^{r^{t-2}}(v_{i,j}^{t}(s))^2\leq  r^{t-3}$ for any $1\leq s\leq n_t$. 
Then 
\begin{align*}
    T&=\sum_{i=1}^r x_i \circ (\sum_{j=1}^{r^{t-2}}u_{i,j}^2\circ \cdots \circ u_{i,j}^t)
    =\sum_{i=1}^r\sum_{j=1}^{r^{t-2}}x_i\circ u_{i,j}^2\circ \cdots \circ u_{i,j}^t,
\end{align*}
 is a decomposition of $T$ with $r^{t-1}$ many components. Here $x_i,\dots, x_r$ are column vectors of $X$ defined in \eqref{eq:T1x}. This gives a decomposition of $T$ into $r^{t-1}$ many rank-one tensors  that satisfies the two inequalities \eqref{eq:a31} and \eqref{eq:a32}. 
This completes the proof.
\end{proof}

With Lemma \ref{lem:gh23}, by the definition of max-quasinorm in Definition \ref{def:max-qnorm}, when $|T|_{\infty}\leq 1$,
\begin{align*}
    \|T\|_{\max}\leq \prod_{i=1}^t \|U^{(i)}\|_{2\to\infty}\leq [r^{(t-1)/2}]^{t-1}\times r^{(t-2)/2}=r^{(t^2-t-1)/2}.
\end{align*}
Then Theorem \ref{thm:max-rank} follows.

\subsection{Proof of Equation \ref{eq:sumsign}}
\label{pf:sumsign}
To see \eqref{eq:sumsign} holds, we prove it for two cases.

(Case 1) 
For any index $(i_1,\dots i_t)\in [n]^t$ that satisfies $S'_{i_1,\dots,i_t}=1$, we know $(s_1)_{i_1}\cdots (s_t)_{i_t}=1$ from \eqref{eq:defS}. 
Then the number of $j$ such that $(s_j)_{i_j}=-1$ is even.  
By our definition of $W_j$ in \eqref{eq:defW}, 
$  \sum_{k=1}^t \mathbf{1}\{i_k\in W_k\}$ is even. 
We can find a corresponding even string $w_j$ for some 
$1\leq j\leq 2^{t-1}$ 
such that for all $1\leq k\leq t$,
\begin{align*}
    (w_j)_k=\begin{cases}
    1 &\text{if } i_k\in W_k,\\
     0 &\text{otherwise}.
    \end{cases}
\end{align*}
Then $i_k\in W_{k,j}$ for all $1\leq k\leq t$. Therefore, the corresponding rank-$1$ tensor satisfies
\[
  \left( 1_{W_{1,j}} \circ\cdots\circ 1_{W_t,j} \right)_{i_1,\dots,i_t}=\prod_{k=1}^t(1_{W_{k,j}})_{i_k}=1.
 \]
On the other hand, all the other rank-$1$ tensors  in the sum of \eqref{eq:sumsign} that do not correspond to the even string $w_j$ will take value $0$ at the entry $(i_1,\dots, i_t)$. 
So in this case, 
\begin{align*}
 S'_{i_1,\dots,i_t} = \sum_{j=1}^{2^{t-1}} 
   \left(1_{W_{1,j}}\circ \cdots \circ 1_{W_{t,j}}\right)_{i_1,\dots,i_t}=1.   
\end{align*}

(Case 2)
If $S'_{i_1,\dots,i_t}=0$, then $(s_1)_{i_1}\cdots (s_t)_{i_t}=-1$, which implies  $\sum_{k=1}^t \mathbf{1}\{i_k\in W_k\}$ is odd and there are no corresponding even strings. 
 Therefore all rank-$1$ tensors on the right hand side of \eqref{eq:sumsign} 
 take value $0$ at $(i_1,\dots,i_t)$, and \eqref{eq:sumsign} holds.

\subsection{Proof of Corollary \ref{cor:noisycompletion}}

\label{pf:noisycompletion}
Let $\hat{T}$ be the solution of Problem \eqref{eq:noisy_problem}. Define the operator 
$\mathcal P_{E}:
\R^{n\times \cdots \times n} \to 
\R^{n\times \cdots \times n}$ 
such that
$(\mathcal P_{E}(T))_{e} = T_e$ 
if $e \in E$ and $0$ otherwise.  Denote $C_t=2^{t} (2t-3)K_G^{t-1}$.
We have 
\begin{align}
  \left| \frac{1}{n^t}\|\hat{T}-T \|_F^2-\frac{1}{|E|}\|\mathcal P_E(\hat{T}-T)\|_F^2     \right|
 &=
 \left| 
 \frac{1}{n^t}\sum_{e\in [n]^t}
 ( \hat{T}_e - T_e )^2 
 - 
 \frac{1}{|E|} \sum_{e \in E}
 ( \hat T_e - T_e )^2
 \right|  \notag\\
 &\leq
    C_t\frac{\lambda}{d}
    \|T \|_\mathrm{max} ^2 \label{eq:noise2},
\end{align}
where \eqref{eq:noise2} follows in the same way as in
\eqref{eq:maxinequality} 
due to the fact that $T$ is a feasible solution to 
\eqref{eq:noisy_problem}. 
On the other hand, 
from the constraints in 
\eqref{eq:noisy_problem} and \eqref{cor:noisebound}, 
by the triangle inequality,
\begin{align}\label{eq:noise3}
    \|\mathcal P_{E}(\hat{T}-T)\|_{F}\leq \|\mathcal P_{E}(\hat{T}-Z)\|_{F}+\|\mathcal P_{E}(Z-T)\|_{F}\leq 2\delta \sqrt{|E|}.
\end{align}
With \eqref{eq:noise2} and \eqref{eq:noise3}, we have 
\begin{align*}
   \frac{1}{n^t}\|\hat{T}-T \|_F^2 &\leq \left| \frac{1}{n^t}\|\hat{T}-T \|_F^2-\frac{1}{|E|}\|\mathcal P_E(\hat{T}-T)\|_F^2     \right| +\frac{1}{|E|}   \|\mathcal P_{E}(\hat{T}-T)\|_{F}^2\\
   &\leq C_t\frac{\lambda}{d}
    \|T \|_\mathrm{max} ^2+4\delta^2,
\end{align*}
which is the final result.

\subsection{Sample Complexity}
\label{sec:sample-complexity}

To compute the sample complexity in the case of no noise in Theorem \ref{thm:main}, 
we would like to guarantee,
from \eqref{eq:errorbound}, that 
\begin{equation}\label{eq:guarantee}
C_t \| T \|_\mathrm{max}^2 \frac{\lambda}{d}
    \leq \varepsilon.
\end{equation}
Then \eqref{eq:guarantee} implies 
\begin{align*}
    \frac{d}{\lambda} 
    \geq 
    \frac{C_t \|T \|_\mathrm{max}^2}{\varepsilon}
      .
\end{align*}
Assuming that $\lambda=O(\sqrt d)$,
i.e.\ $G$ is a good expander, then this gives that
\begin{equation*}
    |E| = n d^{t-1} = 
    O\left( 
    n
    \left( 
    \frac{(C_t \|T \|_\mathrm{max}^2}{\varepsilon}
    \right)^{2(t-1)}
    \right)
\end{equation*}
many samples suffice.
If we assume that $t$ is a constant, then the sample complexity is 
\[ O(\|T\|_{\max}^{4t-4}\varepsilon^{2(1-t)}n).\]
Theorem~\ref{thm:max-rank} gives that
$\|T\|_\mathrm{max}^2 = O(r^{t^2-t-1})$ and
the sample complexity in terms of $\varepsilon, n, r$ can be written as
\[\displaystyle |E| = O\left( \frac{ n r^{2(t-1)(t^2-t-1)}}{\varepsilon^{2(t-1)}} \right).\]
Note that by taking $t=2$ in the expression above, we get the same sample complexity for deterministic matrix completion algorithm as in \cite{heiman2014deterministic}.

\section{Algorithm details}
\label{sec:supp_alg}

Here we detail some components of the numerical algorithm
that are important for its implementation.

\subsection{Variable projection}
\label{sec:var_proj}

One advantage of this relaxed problem
is that we can partially 
minimize out the variable $R$ analytically,
a method referred to as {\it variable projection}.
Problem \eqref{eq:relaxed} is of the form
\begin{equation*}
    \min_{X,R} f(X,R) + g(X) + h(R),
\end{equation*}
where 
$f(X,R) = \frac{\kappa}{2} \|\mask * (X - Z - R)\|_F^2
    + \frac{\beta}{2} \| \mask * R \|_F^2$
are the least squares terms,
$g(X) = \|X\|_\mathrm{max}$, and
$h(R) = \chi \left\{ \|\mask * R \|_F \leq \delta \right\}$
is the indicator function of the constraint set,
equal to $+\infty$ when the constraint is not satisfied
and 0 otherwise.

Let $R^*$ be the minimizer over the $R$ for fixed $X$,
then
\begin{align*}
R^* &:= \arg \min_R f(X,R) + h(R) \\
    &= \mathrm{proj}_{\| \cdot \|_F \leq \delta} 
    \left( 
    \frac{\mask * (X - Z)}{1+\frac{\beta}{\kappa}}
    \right)\\
    &= \mu \mask * (X - Z),
\end{align*}
where
\begin{equation}
    \mu := \left\{ 
    \begin{array}{rcl}
     \frac{1}{1 + \frac{\beta}{\kappa}},
     & &
     \|\mask*(X-Z)\|_F \leq (1 + \frac{\beta}{\kappa} ) \delta 
     ,
     \\
     \frac{\delta}{\| \mask * (X - Z) \|_F} , 
     & &
     \mbox{ otherwise} .
    \end{array}
    \right.
    \label{eq:mu_defn}
\end{equation}
The $\beta$ term is a type of smoothing of the 
hard constraint on the residuals,
which leads to them shrinking by the amount $1 + \frac{\beta}{\kappa}$.

By the results of \citet{aravkin2018},
\eqref{eq:relaxed}
may be solved by replacing $R$ with $R^*$
in the objective function and gradients
and minimizing just over the remaining variable $X$.
Replacing $R$ in the original cost leads to
the projected cost function
\begin{align}
\label{eq:relaxed_value}
    \bar{f}(X) = f(X,R^*) =  
    \frac{1}{2} \left(\kappa (1-\mu)^2 + \mu^2 \beta \right)
    \|\mask * (X - Z)\|_F^2 .
\end{align}
The gradients are obtained simply via
\begin{equation}
\label{eq:grad_var_proj}
    \nabla_X \bar{f}(X) = \nabla_X f(X, R^*)
\end{equation}
without calculating any sensitivities of $R^*$ 
with respect to $X$.

\subsection{Gradients}

For the coordinate descent algorithm, 
we will need to compute the gradient of
the smooth part of the cost with respect
to every factor matrix.
Equation
\eqref{eq:grad_var_proj}
gives
the gradient with respect to $U^{(i)}$
as
\begin{align}
    \nabla_{U^{(i)}} \bar{f}(X(U))
    &=
    \nabla_{U^{(i)}} \frac{\kappa}{2}
    \| \mask * (U^{(1)} \circ \cdots \circ U^{(t)} - Z - R^*) \|_F^2 
    \notag
    \\
    &=
    \nabla_{U^{(i)}} \frac{\kappa}{2}
    \| \mask_{[i]} * 
        ( U^{(i)} K^{(-i)} - Z_{[i]} - R^*_{[i]} ) 
    \|_F^2  
    \notag 
    \\
    &= \kappa \,
    \mask_{[i]} * 
        ( U^{(i)} K^{(-i)} - Z_{[i]} - R^*_{[i]} ) 
        K^{(-i)\intercal} 
    \notag 
    \\
    &= (1-\mu) \kappa \,
    \mask_{[i]} * 
        ( U^{(i)} K^{(-i)} - Z_{[i]}) K^{(-i)\intercal}, 
    \label{eq:grad}
\end{align}
where in the second line we 
matricizing over the $i$th dimension.
The other factors are stored in the Khatri-Rao product
\[
K^{(-i)} := 
U^{(t)} \odot \cdots \odot U^{(i-1)} 
\odot U^{(i+1)} \cdots \odot U^{(1)} .
\]
This gradient can be computed efficiently in two steps.
First, the term
$\mask_{[i]} * ( U^{(i)} K^{(-i)} - Z_{[i]})$
is nothing more than the matricization of the residuals,
which are sparse and can be computed in linear time and
memory.
Second, this must be multiplied by $K^{(-i)\intercal}$,
which is an $r \times n^{t-1}$ matrix that we do not want to form.
There are efficient methods for computing such a
matricized tensor times Khatri-Rao product (MTTKRP)
\citep{bader2007},
and this often dominates the time spent in our algorithm.

The form of the gradient \eqref{eq:grad}
reveals how the variable projection algorithm deals
with the $\beta$-smoothed hard constraint.
If the constraint is not satisfied,
the gradient is proportionally shrunk by an amount $\delta$.
Once the constraint is satisfied, 
the gradient is multiplied by the small but nonzero
amount $\frac{\beta}{\kappa + \beta}$, 
and the operation is continuous at the boundary.
See Figure~\ref{fig:iterations}.
This can be seen as an $\ell_2$ version of soft-thresholding
the gradient.
Without the smoothing, i.e.\ $\beta = 0$, 
the gradient would be exactly
zero once the constraint was satisfied,
and of magnitude $\delta \kappa$ otherwise.
This can be seen as $\ell_2$ hard-thresholding.

\subsection{Proximal operator}

We will also need a way to evaluate
the proximal operator or {\em prox}
of the 2-to-$\infty$ norm of a matrix:
\[
\mathrm{prox}_{s \|\cdot \|_{2,\infty}} (X)
:=
\arg \min_Y \frac{1}{2} \|Y - X\|_F^2 + s \| Y \|_{2,\infty} .
\]
Since $\|\cdot\|_{2,\infty}$ 
is a norm, the Moreau decomposition means
we may use the projection onto
its dual norm ball to compute
\begin{equation}
\label{eq:prox_moreau}
\mathrm{prox}_{s \|\cdot \|_{2,\infty}} (X)
= X - 
s \cdot \mathrm{proj}_{\| \cdot \|_* \leq 1} \left( \frac{X}{s} \right) .
\end{equation}
The dual norm is defined as
\begin{align*}
\| X \|_* 
:= \max_{\|Z \|_{2,\infty} \leq 1} \langle Z, X \rangle   
= \max_{Z: \max_i \|z_i\|_2 \leq 1} \sum_i z_i^\intercal x_i 
= \sum_i \| x_i \|_2,
\end{align*}
where $x_i$ and $z_i$ are the $i$th rows of $X$ and $Z$.
The last equality comes from taking $z_i = x_i / \| x_i \|_2$.
This dual norm is of course the $\ell_{2, 1}$ induced norm
and known to induce group sparsity. 
\citet{liu2009}
found the following characterization of the projection.
\begin{lemma}[\citet{liu2009}, Theorems 5--6]
\label{lem:dual_projection}
If $\| X \|_* > r$, then 
$W := \mathrm{proj}_{\| \cdot \|_* \leq r}(X)$
is given by
\begin{equation}
\label{eq:dual_proj}
w_i = \left\{ 
\begin{array}{cll}
     \left(1 - \frac{\bar\lambda}{\|x_i\|_2}\right) x_i &, & 
     \|x_i\| > \bar\lambda  \\
     0 &, &  \|x_i \| \leq \bar\lambda
\end{array}
\right.  ,
\end{equation}
where $x_i$ and $w_i$ are the $i$th rows of $X$ and $W$,
and $\bar\lambda$ is the unique root of
\begin{equation}
\label{eq:dual_root}
\omega(\lambda) = \sum_{i} \max(\|x_i\| - \lambda, 0) - r.
\end{equation}
Otherwise, if $\|X\|_* \leq r$, then $W = X$.
\end{lemma}

Lemma~\ref{lem:dual_projection} provides an efficient routine
for computing $\mathrm{prox}_{s \| \cdot \|_{2,\infty}}(X)$.
First, \eqref{eq:prox_moreau} is used to express
the prox in terms of the projection onto the dual norm ball.
Since $\bar\lambda$ satisfies
$0 \leq \bar\lambda \leq \frac{1}{s} \|X \|_{2,\infty}$,
a bracketing method may be used to solve the
scalar nonlinear equation \eqref{eq:dual_root} 
for the root $\bar\lambda$.
Then \eqref{eq:dual_proj} is used to return the projection
needed to evaluate the prox.

\subsection{Coordinate descent}

For completeness, Algorithm~\ref{alg:coord_descent}
shows the pseudocode for solving the relaxed max-qnorm
optimization problem.

\begin{algorithm}[t!]
\caption{Coordinate descent method}
\label{alg:coord_descent}
\begin{algorithmic}[1]
\STATE{Input:
initialization
$\{U^{(i)}\}_{i=1}^t$,
parameters
$0 < \kappa < \infty$ ,
$0 < \beta < \infty$,
$\delta > 0$}
\REPEAT
\FOR{$i = 1, \ldots, t$}
\STATE{ 
$U^{(i)} 
\leftarrow
\arg \min_{U} 
C(U^{(1)}, \ldots, U^{(i-1)}, U, U^{(i+1)}, \ldots, U^{(t)})$ 
\hfill 
accelerated prox gradient}
\ENDFOR
\IF{first iteration}
\STATE{rescale factors so that $\|U_{:,i}^{(1)} \|_2 = \ldots = \| U_{:,i}^{(t)} \|_2$ 
for $i=1,\ldots,r$}
\ENDIF
\UNTIL{convergence criteria}
\RETURN{$\{\hat U^{(i)}\}_{i=1}^t$}
\end{algorithmic}
\end{algorithm}

\subsection{Additional plots}

The results for 
experiments described in Section~\ref{sec:experiments}
with $t=3$ are shown in
Figure~\ref{fig:dim_3_results}.

\begin{figure}[ht!]
    \centering
    \includegraphics[width=.8\linewidth,trim={20 0 20 0},
    clip]{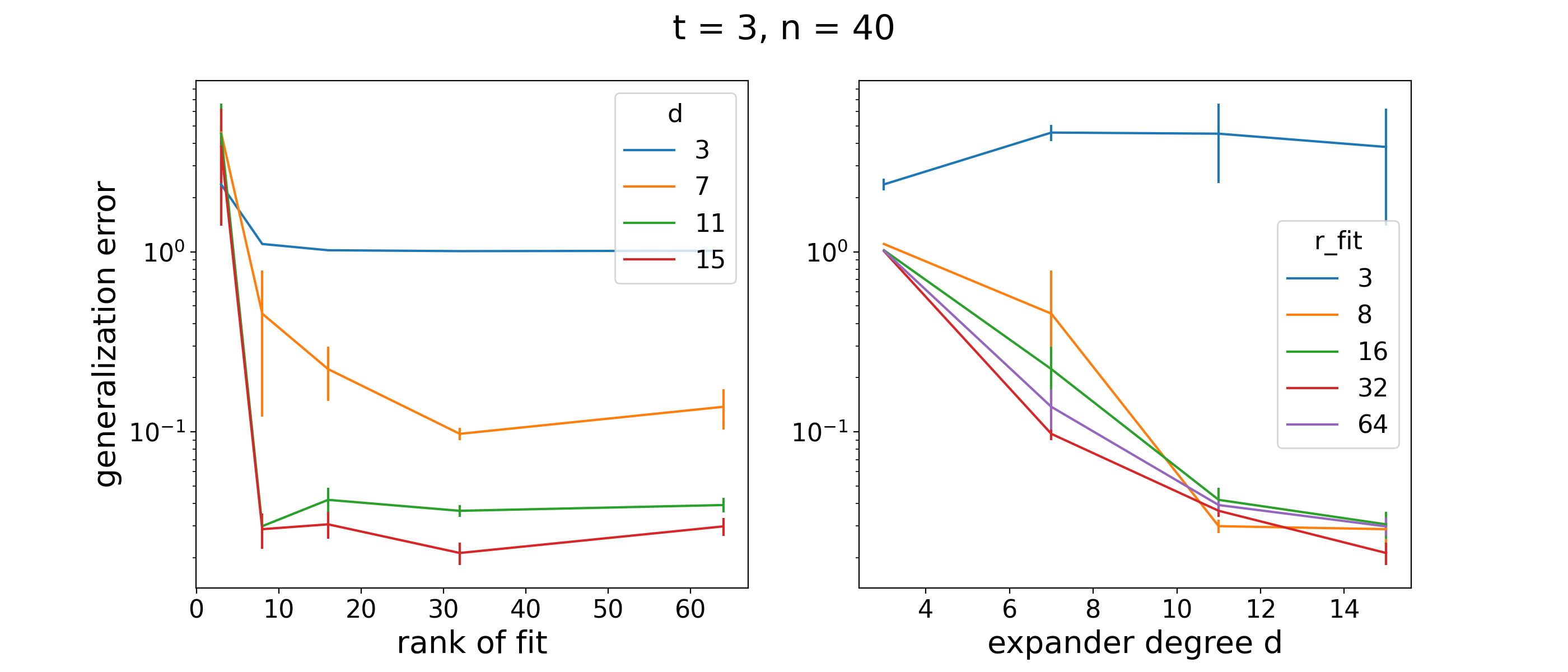}
    
    \includegraphics[width=.8\linewidth,trim={20 0 20 0},
    clip]{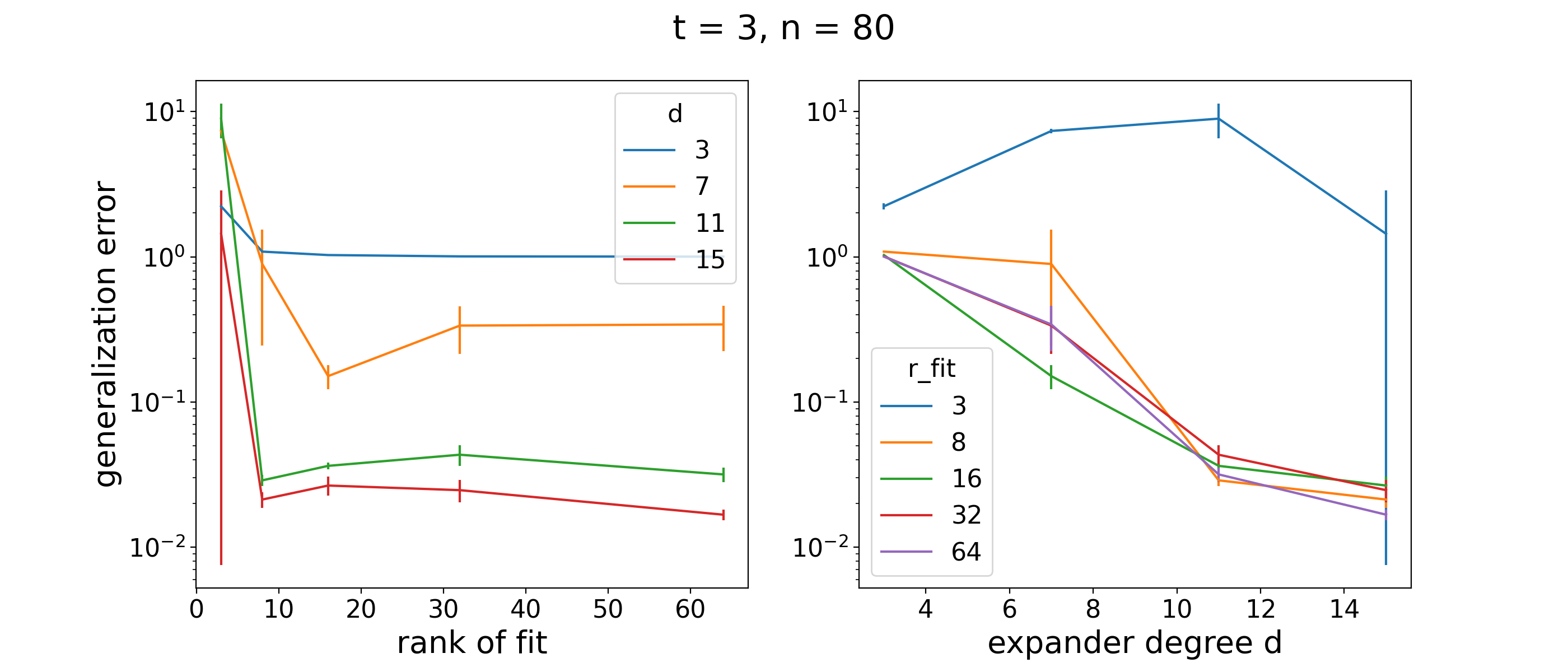}
    \caption{Relative error of reconstruction for $t=3$ and $n=40,80$.
        The results are average of 6 tensors with standard errors shown.}
    \label{fig:dim_3_results}
\end{figure}

\end{document}